\documentclass[10pt]{article}

\usepackage[letterpaper,margin=1in]{geometry}
\usepackage[parfill]{parskip}
\usepackage[authoryear, compress]{natbib}

\usepackage[utf8]{inputenc} 
\usepackage[T1]{fontenc}    
\usepackage[colorlinks,citecolor=blue,urlcolor=blue,linkcolor=blue,linktocpage=true]{hyperref}
\usepackage{url}            
\usepackage{booktabs}       
\usepackage{amsfonts}       
\usepackage{nicefrac}       
\usepackage{microtype}      
\usepackage[parfill]{parskip}
\usepackage{amsmath,amsthm,amssymb,bbm}
\usepackage{mathtools}
\usepackage{cases}
\usepackage{comment}
\usepackage{subfigure}
\usepackage{algorithm,algorithmic}
\usepackage{color}
\usepackage{appendix}
\usepackage{xspace}
\usepackage{enumitem}
\usepackage[english]{babel}
\usepackage{cleveref}
\crefformat{equation}{(#2#1#3)}
\crefrangeformat{equation}{(#3#1#4) to~(#5#2#6)}
\crefname{equation}{}{}
\Crefname{equation}{}{}

\crefname{definition}{\textbf{definition}}{definitions}
\Crefname{definition}{Definition}{Definitions}
\crefname{assumption}{\textbf{assumption}}{assumptions}
\Crefname{assumption}{Assumption}{Assumptions}
\definecolor{maroon}{RGB}{192,80,77}

\newcommand{\explain}[2]{\underset{\mathclap{\overset{\uparrow}{#2}}}{#1}}

\newcommand\independent{\protect\mathpalette{\protect\independenT}{\perp}}
\def\independenT#1#2{\mathrel{\rlap{$#1#2$}\mkern2mu{#1#2}}}

\newtheorem{theorem}{Theorem}
\newtheorem{lemma}[theorem]{Lemma}
\newtheorem{fact}[theorem]{Fact}
\newtheorem{proposition}[theorem]{Proposition}

\newtheorem{corollary}[theorem]{Corollary}
\newtheorem{definition}[theorem]{Definition}

\newtheorem{remark}[theorem]{Remark}

	\AtEndDocument{\refstepcounter{theorem}\label{finalthm}}

\newcommand{\OBP}{\textsc{ObjPert}}
\newcommand{\OPS}{\textsc{OPS}}
\newcommand{\AdaOPS}{\textsc{AdaOPS}}

\newcommand{\R}{\mathbb{R}}

\newcommand{\argmin}{\mathop{\mathrm{argmin}}}

\def\E{\mathbb{E}}
\def\P{\mathbb{P}}

\def\Var{\mathrm{Var}}

\def\tr{\mathrm{tr}}

\def\diag{\mathrm{diag}}

\def\htheta{\hat{\theta}}
\def\ttheta{\tilde{\theta}}

\def\R{\mathbb{R}}

\def\cA{\mathcal{A}}
\def\cB{\mathcal{B}}
\def\cD{\mathcal{D}}

\def\cH{\mathcal{H}}

\def\cN{\mathcal{N}}
\def\cP{\mathcal{P}}

\def\cS{\mathcal{S}}

\def\cX{\mathcal{X}}
\def\cY{\mathcal{Y}}
\def\cZ{\mathcal{Z}}

\title{Per-instance Differential Privacy}

\author{
	Yu-Xiang Wang\thanks{Corresponding email: \texttt{yuxiangw@cs.ucsb.edu}. The bulk of this manuscript was written when the author was a PhD student in Carnegie Mellon University.} \\
UC Santa Barbara \\
Santa Barbara, CA 93106\\
}

\date{ }
\begin{document}
	
	\maketitle
	
	\begin{abstract}
		We consider a refinement of differential privacy --- per instance differential privacy (pDP), which captures the privacy of a specific individual with respect to a fixed data set.  We show that this is a strict generalization of the standard DP and inherits all its desirable properties, e.g.,  composition, invariance to side information and closure to postprocessing, except that they all hold for every instance separately. 
		When the data is drawn from a distribution, we show that moments of per-instance DP imply generalization. Moreover, we provide explicit calculations of the per-instance DP for the output perturbation on a class of smooth learning problems. The result reveals an interesting and intuitive fact that an individual has stronger privacy if he/she has small ``leverage score'' with respect to the data set and if he/she can be predicted more accurately using the leave-one-out data set. Simulations show several orders-of-magnitude more favorable privacy and utility trade-off when we consider the privacy of only the users in the data set. In a case study on differentially private linear regression, we provide a novel analysis of the One-Posterior-Sample (OPS) estimator and show that when the data set is well-conditioned it provides $(\epsilon,\delta)$-pDP for any target individuals and matches the exact lower bound up to a $1+\tilde{O}(n^{-1}\epsilon^{-2})$ multiplicative factor.  We also demonstrate how we can use a ``pDP to DP conversion'' step to design AdaOPS which uses adaptive regularization to achieve the same results with $(\epsilon,\delta)$-DP.
	\end{abstract}
	
	
	\section{Introduction}

While modern statistics and machine learning had seen amazing success, their applications to sensitive domains involving personal data remain challenging due to privacy issues.
Differential privacy (DP) \citep{dwork2006calibrating} is a mathematical definition that provides strong provable protection to individuals and prevents them from being identified by an arbitrarily powerful adversary. DP has been increasingly popular within the machine learning community as a solution to the aforementioned problem \citep{mcsherry2009differentially,chaudhuri2011differentially,liu2015fast,abadi2016deep}. 
The strong privacy protection, however, comes with a steep price to pay.  {Differential privacy often leads to a substantial and sometimes unacceptable drop in utility,} e.g., in contingency tables \citep{fienberg2010differential} and in genome-wide association studies \citep{yu2014differentially}. 
This motivated a large body of research to focus on making differential privacy more practical \citep{nissim2007smooth,dwork2009differential,Sheffet2015differentially,wang2015privacy,dwork2016concentrated,bun2016concentrated,foulds2016theory} by exploiting local structures and/or revising the privacy definition.

The majority of these approaches adopt a ``privacy-centric'' model, which
involves theoretically proving that an algorithm is differentially private for any data sets (within a data domain), then carefully analyzing the utility of the algorithm under additional assumptions on the data set. For instance, in statistical estimation, it is often assumed that the data is drawn i.i.d. from a family of distributions. In nonparametric statistics and statistical learning, the data are often assumed to having specific deterministic/structural conditions, e.g., smoothness, incoherence, eigenvalue conditions, low-rank, sparsity and so on.  While these assumptions are strong and sometimes unrealistic, they are often necessary for a model to work correctly, even without privacy constraints. Take high-dimensional statistics for example, ``sparsity'' is never really true in applications, but if the true model is dense and unstructured, then information-theoretically no methods can estimate the true model anyway in the ``small $n$ large $d$'' regime. 
 {That is why \citet{friedman2001elements} argued that one should ``bet on sparsity'' regardless and only expect the method to work when the input data is well-approximated by a sparse model.
	Behind this informal ``bet on sparsity'' principle, is the pursuit for \emph{adaptivity} in modeling and algorithm design.
We say an algorithm is \emph{adaptive}\footnote{ {This is not to be confused with the ``adaptivity'' as in ``adaptive composition''\citep{dwork2010boosting} and ``adaptive data analysis''\citep{dwork2014preserving} commonly seen in the differential privacy literature. The latter is about how a sequence of actions can be chosen as a function of the outcomes to all previous actions, while the notion of ``adaptivity'' that we considered is the same as that in ``adaptive estimation'', ``adaptive algorithm design'' \citep[see, e.g.,][]{bickel1982adaptive}, which is about the extent to which algorithms can exploits properties in the data sets without knowing that they exists a priori.}} if it can automatically \emph{adapt} to favorable properties of each input data set and perform better.}




 {
These conditions on the data set can also have a profound impact on a DP algorithm's privacy guarantee. 
As we know, DP quantifies its privacy guarantee with a single nonnegative number $\epsilon$ --- the privacy loss. Smaller $\epsilon$ implies stronger privacy guarantee. DP algorithms are designed to calibrate itself according to a prescribed budget of $\epsilon$, and to achieve $\epsilon$-DP (or $(\epsilon,\delta)$-DP) regardless of what the input data is. 
This concise data-independent privacy loss $\epsilon$ is one of DP's most attractive feature as it makes DP universally applicable.
But in many real-world applications, $\epsilon$ is often an \emph{overly simplified summary} and a \emph{crude upper bound} of the \emph{actual} privacy loss incurred to  individuals in the data set.
An algorithm $\cA$ that is calibrated to achieve a privacy loss of $\epsilon=10$ on the worst pair of adjacent data sets, could imply a much stronger privacy level of $\epsilon' = 0.1$ when $\cA$ is applied to a particular data set in which some of the additional assumptions are true. 
In this case, it will be too conservative to quantify the actual privacy loss with just $\epsilon=10$.
}




The extent to which DP is conservative, however, is highly problem-dependent. In cases such as releasing counting queries, DP's $\epsilon$ clearly measures the correct information leakage, since the sensitivity of such queries do not change with respect to the two adjacent data sets; however, in the context of machine learning and statistical estimation (as we will show later), the $\epsilon$ of DP can be orders of magnitude larger than the actual amount of information leakage that comes with the release. That is part of the reason why in practice, it is challenging even for experts of differential privacy to provide a consistent recommendation on standard questions such as: 
\begin{center}
	\textsf{
		``What is the value of privacy budget $\epsilon$ I should set in my application?''}
\end{center}

{ 
In this paper, we take a new ``algorithm-centric'' approach of analyzing privacy and address a related but different question:
\begin{center}
	\textsf{``What is the privacy loss $\epsilon$ incurred to an individual $z$ when $\cA(Z)$ is released?''}
\end{center}
}
 Instead of designing algorithms that take the privacy budget $\epsilon$ as an input, we  start with a fixed randomized algorithm $\cA$\footnote{$\cA$ can be a DP algorithm but it does not have to be.}  and then analyze its privacy protection for every pair of adjacent data sets separately. This is equivalent to treating $\epsilon$ as a function parameterized by each problem instance --- a Dataset-Target pair. $\epsilon$ as a function provides a more fine-grained description of the randomized algorithm compared to using only the privacy loss $\epsilon$ of its DP guarantee, i.e., the maximum of $\epsilon(\text{Dataset}, \text{Target})$ over all pairs of datasets and individuals.


Our contribution is threefold. 
\begin{enumerate}
	\item First, we develop per-instance differential privacy as a strict generalization of the standard pure and approximate DP. It provides a more fine-grained description of the privacy protection for each target individual and a fixed data set. We show that it inherits many desirable properties of differential privacy and	can easily recover differential privacy for a given class of data and target users.
	
	\item Secondly, we quantify the per-instance sensitivity in a class of smooth learning problems including linear and kernel machines. The result allows us to explicitly calculate per-instance DP of a multivariate Gaussian mechanism. For an appropriately chosen noise covariance, the per-instance DP is proportional to the norm of the pseudo-residual in norm specified by the Hessian matrix. In particular, in linear regression, the per-instance sensitivity for a data point is proportional to its square root statistical leverage score (predictive variance) and its leave-one-out prediction error (predictive bias).
	\item Lastly, we analyze the procedure of releasing one sample from the posterior distribution (the OPS estimator) for linear and ridge regression as an output perturbation procedure with a data-dependent choice of the covariance matrix. We show using the pDP technique that, when conditioning on a data set drawn from the linear regression model or having a well-conditioned design matrix, OPS achieves $(\epsilon,\delta)$-pDP for  while matching the Cramer-Rao lower bound up to a $1+\tilde{O}(n^{-1}\epsilon^{-2})$ multiplicative factor. OPS, unfortunately, cannot achieve DP with a constant $\epsilon$ while remaining asymptotically efficient. We fixed that by a new algorithm called \AdaOPS{}, which provides $(\epsilon,\delta)$-DP and $1+\tilde{O}(n^{-1}\epsilon^{-2})$-statistical efficiency at the same time.

\end{enumerate}

{ 
To avoid any confusion, we also highlight a few things that this paper is \emph{not} about.
First of all, pDP is \emph{not} a replacement of DP. It is rather an analytical tool for us to understand the adaptivity in privacy loss and to design more data-dependent DP algorithms. For instance, you can use pDP to describe the \emph{actually incurred privacy loss} to an individual when an $\epsilon$-DP algorithm is applied to a given data set, rather than just covering everything under a blanket statement that: ``All I know is that it's smaller than $\epsilon$.''
Secondly, for output perturbation algorithms, we do \emph{not} advocate calibrating the noise to per-instance sensitivity for setting pDP to a prescribed budget, because the per-instance sensitivity itself is a data-dependent quantity. This is not an intended use for pDP. 
In fact, if you insist on doing that, then you essentially changed the algorithm all together and the corresponding per-instance sensitivity will change as well.  
Thirdly, pDP privacy loss is something that the curator can calculate and keep as a confidential certificate, but \emph{not} shared publicly or even with data contributors, since pDP itself contains private information about the entire data set. Publishing pDP differentially privately is an important problem and it is part of an ongoing future work.

%
}

\subsection{Symbols and notations}
Throughout the paper, we will use the standard notation in statistical learning.
Data point $z\in \cZ$. In supervised learning setting, $z=(x,y)\in \cX\times\cY = \cZ$. We use $\theta\in \Theta$ to denote either the predictive function $\cX\rightarrow \cY$ or the parameter vector that specifies such a function. 
$\ell: \Theta \times \cZ\rightarrow \R$ to denote the loss function or in a statistical model, $\ell$ represents the negative log-likelihood $-\log p_{\theta}(z)$. For example, in linear regression, $\cX \subset \R^d$, $\cY\subset \R$, $\Theta\subset \R^d$ and $\ell(\theta,(x,y)) = (y-x^T\theta)^2$. 
{ 
Capital $Z$ denotes a data set of an unspecified size, i.e., $Z\in \cZ^{*} =\cup_{n=0,1,2,3,...} \cZ^n$.   We use $\cA: \cZ^{*} \rightarrow P_{\Theta}$ to denote a randomized algorithm that takes in a data set and outputs a draw from a distribution defined on a model space. In particular, $\cA(Z)$ is used to denote both a random variable and its distribution, so that we can say $\theta\sim \cA(Z)$.
$\epsilon$ and $\epsilon(Z,z)$ will be reserved to denote privacy loss, and $Z,Z'\in \cZ^{*}$ are reserved to denote the two adjacent data set.  In particular, unless we specify otherwise, $Z'$ will be either adding $z$ to $Z$ or removing $z$ from $Z$ depending on whether $Z$ contains $z$ or not.  The notation $Z \overset{z}{\sim} Z'$ is used to explicitly say that $Z$ and $Z'$ differ by a single data point $z$.
}

{ 
\subsection{Related work}\label{sec:related}
This paper is related to the existing work in relaxing DP\citep{hall2013random,barber2014privacy,wang2016average, dwork2016concentrated,bun2016concentrated,mironov2017renyi}, personalizing DP\citep{ghosh2015selling,ebadi2015differential,liu2015fast}, post hoc calculation of privacy guarantee \citep{abadi2016deep,rogers2016privacy,ligett2017accuracy,balle2018improving}, as well as in analytical frameworks for designing data-adaptive DP algorithms \citep{nissim2007smooth,dwork2009differential}. We provide more details below.

\paragraph{Relaxing and personalizing DP.}
Recall that the pure differential privacy loss can be defined as
$$
\epsilon = \sup_{Z\overset{z}{\sim}Z'} \sup_{\theta\in \Theta} \log \frac{p_{\cA(Z)}(\theta)}{p_{\cA(Z')}(\theta)}.
$$
The effort in relaxing differential privacy mostly consider relaxing the $\sup_{\theta}$ part of the definition, by either using a different divergence measure \citep{barber2014privacy} or explicitly treating $\epsilon(\theta) = \log \frac{p_{\cA(Z)}(\theta)}{p_{\cA(Z')}(\theta)}$ as a random variable induced by $\theta \sim \cA(Z)$. The privacy random variable point of view connects $(\epsilon,\delta)$-DP to concentration inequalities and in particular, it produces the advanced composition of privacy losses via Martingale concentration \citep{dwork2010boosting}. More recently, the idea is extended to define weaker notions of privacy such as concentrated-DP~\citep{dwork2016concentrated,bun2016concentrated} and R\'enyi-DP~\citep{mironov2017renyi}. They allow for tighter accounting of the privacy losses through the moment generating function of the privacy random variable.

Our work is complementary to this line of work, as we relax the $ \sup_{Z\overset{z}{\sim}Z'} $ part of the definition and consider the adaptivity of $\epsilon$ to a fixed pair of data set $Z$ and privacy target $z$. In some cases, we consider $\epsilon$ to be a random variable jointly parameterized by $Z,z$ and $\theta$. 

The closest existing definition to ours is perhaps the personalized-DP, first seen in \citet{ghosh2015selling} for the problem of selling privacy in auctions and reinvented by \citet{ebadi2015differential,liu2015fast} in the context of private database queries and private recommendation systems respectively. They also try to capture a personalized level of privacy for each individual $z$. The difference is that personalized-DP considers adding or removing $z$ from all data sets, while we consider adding $z$ or removing $z$ from a fixed $Z$.

Finally, pDP is related to random differential privacy \citep{hall2013random} and on-average KL-privacy \citep{wang2016average}. They respectively measure the high-probability and expected privacy loss when $z$ and the data points in $Z$ are drawn i.i.d. from a distribution $\cD$, while we consider a fixed $(Z,z)$ pair that is not necessarily random. 

We summarize these definitions in Table~\ref{tab:compare_privacy}. It is clear from the table that if we ignore the differences in the probability metric used, per-instance DP is arguably the most general, and adaptive, since it depends on specific $(Z,z)$ pairs. 
\begin{table}[ht]
	\centering
		\caption{Comparing variants of differential privacy.}\label{tab:compare_privacy}
	\begin{tabular}{c|ccc|c }
		& Data set & private target & probability metric  & parametrized by\\\hline
		Pure-DP
		& $\sup_Z$& $\sup_z$ & $D_{\infty}(P\|Q)$ & $\cA$ only\\
		Approx-DP
		& $\sup_Z$& $\sup_z$ & $D_{\infty}^{\delta}(P\|Q)$ & $\cA$ only\\
		(z/m)-CDP
		&$\sup_Z$ &$\sup_z$ &$D_{\text{subG}}(P\|Q)$ &$\cA$ only\\
		R\'enyi-DP
		& $\sup_Z$& $\sup_z$& $D_{\alpha}(P\|Q)$ &$\cA$ only\\
		Personal-DP
		&$\sup_Z$& fixed $z$ & $D_{\infty}^{\delta}(P\|Q)$ & $\cA$ and $z$\\
		TV-privacy
		&  $\sup_Z$& $\sup_z$ & $\|P-Q\|_{TV}$ & $\cA$ only\\
		KL-privacy
		& $\sup_Z$& $\sup_z$ & $D_{KL}(P\|Q)$ & $\cA$ only\\
		On-Avg KL-privacy
		& $\E_{Z\sim\cD^n}$ & $\E_{z\sim \cD}$ & $D_{KL}(P\|Q)$ & $\cA$ and $\cD$\\
		Random-DP 
		& $1-\delta$ & $1-\delta$ & $D_{\infty}^{\delta}(P\|Q)$ & $\cA$ and $\cD$\\
		Per-instance DP& fixed $Z$& fixed $z$ & $D_{\infty}^{\delta}(P\|Q)$& $\cA$, $Z$ and $z$\\
	\end{tabular}
\end{table}

%
%
%

\paragraph{Post hoc calculation of privacy loss.}
The idea of calculating the privacy loss after running a fixed randomized algorithm is not new. It is the inverse problem of the typical task of calibrating noise to meet a prescribed privacy requirement and is often used as an intermediate step in the analysis of the latter \citep{dwork2006calibrating}. 

More recently,  the post hoc view is adopted in the design of privacy odometer that tracks the post hoc overall privacy loss of a list of sequentially-chosen privacy parameters \citep{rogers2016privacy}. Their analysis stays at an abstract-level as it describes an algorithm solely by its $(\epsilon,\delta)$-DP guarantee.
It is also used in a more refined algorithm-specific privacy analysis for noisy SGD \citep{abadi2016deep} and Gaussian noise adding \citep{balle2018improving}, which simultaneously ensures $(\epsilon(\delta),\delta)$-DP for all $0<\delta<1$ with a monotonically decreasing function $\epsilon(\delta)$. However, they do not adapt to the given input data set.

\citet{ligett2017accuracy} defines``ex-post privacy loss'' (Definition 2.2.) to be the realized privacy loss random variable $\epsilon(\text{Outcome})$, but also do not adapt to the given input data set.  In fact, a direct comparison of their analysis of linear/ridge regression (Theorem 3.1) to our case study reveals that our pDP analysis with the subsequent pDP-to-DP conversion allows us to come up with an algorithm that exploits the strong convexity that comes from the data set, and hence a more favorable bias-variance trade-off.

Data-dependent post hoc privacy analysis relatively recent and was discussed in \citep{papernot2016semi} in the same flavor of ``pDP for all'', except that it is done with Renyi DP. They did not consider more fine-grained pDP which can be different for every individual.

\paragraph{Frameworks for data-dependent DP algorithms.}
Data-dependent DP algorithms were investigated under the classical framework of smooth sensitivity \citep{nissim2007smooth} and propose-test-release (PTR) \citep{dwork2009differential}. 
The focus of these popular frameworks is on how to calibrate noise to local sensitivity, rather than how to calculate the data-dependent privacy loss after running a fixed randomized algorithm. Note that the algorithm under consideration needs not be differentially private and we do not propose to calibrate an algorithm based on pDP. The purpose of pDP analysis is to provide a more precise privacy loss summary of a given randomized algorithm, even though it might not be DP in the worse case.

Building upon the pDP analysis, we demonstrated that sometimes one can use it to design data-dependent DP algorithm. The approach is closely related to the PTR framework but has a more systematic way of identifying the key quantities that contribute to the sensitivity. Also, instead of proposing and testing an (often exponentially long) sequence of criteria, our approach involves directly releasing differentially private high-confidence bounds of certain key quantities. In particular, the proposed data-dependent differentially private linear regression estimator by \citep{dwork2009differential} runs in time exponential in the dimension, while our proposed \AdaOPS is polynomial in all parameters.

While the manuscript is under peer-review, an anonymous reviewer brought to our attention the independent work of \citet{cummings2018individual}. They consider an alternative framework for data-dependent DP algorithm design in which they use a notion called ``individual sensitivity'', which measures the maximum perturbation of a function when we add an individual $z$ to any data set $Z$ that does not contain $z$. This is different from either local sensitivity or per-instance sensitivity (that we will discuss in the next section) but is almost identical to the ``personalized sensitivity'' used in \citep{ghosh2015selling,liu2015fast,ebadi2015differential}. For many problems, e.g., linear regression, personalized sensitivity can be unbounded for all $z$, while per-instance sensitivity remains finite provided that the design matrix is not singular.
In addition, their approach runs in exponential time in general and seems to apply only to output perturbation algorithms. pDP, on the other hand, is well-defined for any randomized algorithm.
}
	\section{Per-instance differential privacy}
In this section, we define per-instance differential privacy, and derive its properties. We begin by parsing the standard definition of differential privacy.

\begin{definition}[Differential privacy \citep{dwork2006calibrating}]
		We say a randomized algorithm $\cA$ satisfies $(\epsilon,\delta)$-DP if, for \emph{all} data set $Z$ and data set $Z'$ that can be constructed by adding or removing one data point $z$ from $Z$, 
	$$
		\P_{\theta\sim \cA(Z)}(\theta\in \cS)  \leq e^{\epsilon} \P_{\theta\sim \cA(Z')}(\theta\in \cS) + \delta,\;\;\forall \text{ measurable set } \cS.
	$$                                                  
\end{definition}
When $\delta=0$, this is also known as pure differential privacy.

It is helpful to understand what differential privacy is protecting against --- a powerful adversary that knows everything in the entire universe, except one bit of information: whether a target $z$ is in the data set or not in the data set. The optimal strategy for such an adversary is to conduct a likelihood ratio test (or posterior inference) on this bit, and differential privacy uses randomization to limit the probability of success of such test \citep{wasserman2010statistical}. 



Note that the adversary always knows $Z$ and has a clearly defined target $z$, and it is natural to evaluate the winnings and losses of the ``player'', the data curator by conditioning on the same data set and privacy target. This gives rise to the following generalization of DP.

\begin{definition}[Per-instance Differential Privacy]
	For a fixed data set $Z$ and a fixed data point $z$. We say a randomized algorithm $\cA$ satisfy $(\epsilon,\delta)$-per-instance-DP for $(Z,z)$ if, for all measurable set $S\subset\Theta$, it holds that
	\begin{align*}
	P_{\theta\sim \cA(Z)}(\theta\in S)  \leq e^{\epsilon} P_{\theta\sim \cA([Z,z])}(\theta\in S) + \delta,\\
	P_{\theta\sim \cA([Z,z])}(\theta\in S) \leq e^{\epsilon} P_{\theta\sim \cA(Z)}(\theta\in S) + \delta.
	\end{align*}
\end{definition}
This definition is different from DP primarily because DP is the property of the $\cA$ only and pDP is the property of both $\cA, Z$ and $z$. If we take supremum over all $Z\in \cZ^n$ and $z\in \cZ$, then it recovers the standard differential privacy. 

Similarly, we can define per-instance sensitivity for $(Z,z)$.
\begin{definition}[per-instance sensitivity]
	Let $\cH = \R^d$, for a fixed $Z$ and $z$. The per-instance $\|\cdot\|_*$ sensitivity of a function $f: \text{Data} \rightarrow \R^d$ is defined as
	$\|f(Z) - f([Z,z])\|_{*}$, where $\|\cdot\|_*$ could be $\ell_p$ norm or $\|\cdot\|_A=\sqrt{(\cdot)^T A (\cdot)}$ defined by a positive definite matrix $A$.
\end{definition}
This definition also generalizes quantities in the classic DP literature. If we fix $Z$ but maximize over all $z\in Z$, we get local-sensitivity \citep{nissim2007smooth}. If we maximize over both $Z\in \cZ^*$ and $z\in \cZ$, we get global sensitivity \citep[Definition 3.1]{dwork2014algorithmic}. These two are often infinite in real-life problems, but for a fixed data set $Z$ and target $z$ to be protected, we could still get meaningful per-instance sensitivity.

Immediately, the per-instance sensitivity implies pDP for a noise adding procedure.
\begin{lemma}[Multivariate Gaussian mechanism]\label{lem:multigaussianDP}
	Let $\htheta$ be a deterministic map from a data set to a point in $\Theta$, e.g., a deterministic learning algorithm, and let the $A$-norm per-instance sensitivity $\Delta_A(Z,z)$ be $\|\hat{\theta}([Z,z]) -\hat{\theta}(Z) \|_A$.
	Then adding noise with covariance matrix $A^{-1}/\gamma$ obeys $(\epsilon(Z,z),\delta)$-pDP for any $\delta>0$ with
	$$
	\epsilon(Z,z) = \gamma \Delta_A(Z,z)\sqrt{\log(1.25/\delta)}.
	$$
\end{lemma}
The proof, which is standard and we omit, simply verifies the definition of $(\epsilon,\delta)$-pDP by calculating a tail bound of the privacy loss random variable and invokes Lemma~\ref{lem:tailbound2DP}.

	\subsection{Basic properties of pDP}
We now describe properties of per-instance DP, which mostly mirror those of DP.
\begin{fact}[Strong protection against identification]
	Let $\cA$ obeys $(\epsilon,\delta)$-pDP for $(Z,z)$, then for any measurable set  $\cS\subset \Theta$ where $\min\{\P_{\theta\sim\cA(Z)}(\theta\in \cS),\P_{\theta\sim\cA(Z)}(\theta\in \cS)\} \geq \delta/\epsilon$
	then given any side information $\textsf{aux}$
	$$
	-2\epsilon \leq \log \frac{\P_{\theta\sim \cA(Z)}(\theta\in \cS | \textsf{aux})}{\P_{\theta \sim \cA([Z,z])}(\theta\in \cS | \textsf{aux})}\leq 2\epsilon.
	$$
\end{fact}
\begin{proof}
	Note that after fixing $Z$, $\theta$ is a fresh sample from $\cA(Z)$, as a result, $\theta\independent \textsf{aux} | Z$. The claimed fact then directly follows from the definition.
\end{proof}
Note that the log-odds ratio measures how likely one is able to tell one distribution from another based on side information and an event $\cS$ of the released result $\theta$. When the log-odds ratio is close to $0$, the outcome $\theta$ is equally likely to be drawn from either distribution.

\begin{fact}[Convenient properties directly inherited from DP]
	For each $(Z,z)$ separately we have:
	\begin{enumerate}
		\item Simple composition:
		Let $\cA$ and $\cB$ be two randomized algorithms, satisfying $(\epsilon_1,\delta_1)$-pDP, $(\epsilon_2,\delta_2)$-pDP, then $(\cA,\cB)$ jointly is $(\epsilon_1+\epsilon_2,\delta_1+\delta_2)$-pDP.
		\item Advanced composition:
		Let $\cA_1,...,\cA_k$ be a sequence of randomized algorithms, where $\cA_i$ could depend on the realization of $\cA_1(Z),...,\cA_i(Z)$, each with $(\epsilon,\delta)$-pDP, then jointly $\cA_{1:k}$ obeys $O(\sqrt{k\log (1/\delta)}\epsilon),O(k\delta)$-pDP.  The same claim also holds for algorithm-specific advanced composition via concentrated DP and Renyi DP.
		\item Closedness to post-processing:
		If $\cA$ satisfies $(\epsilon_1,\delta_1)$-pDP, for any function $f$, $f(\cA(\cdot))$ also obeys $(\epsilon_1,\delta_1)$-pDP.
		\item Group privacy:
		If $\cA$ obeys $(\epsilon,\delta)$-pDP with $\epsilon,\delta$ parameterized by $(\mathrm{Data},\mathrm{Target})$, then 
		\begin{align*}
		P_{\theta\sim \cA(Z)}(\theta\in S) \leq  e^{\epsilon(Z,z_1)+\epsilon([Z,z_1],z_2)+...+\epsilon([Z,z_{1:k-1}],z_k)}  P_{\theta\sim \cA([Z,z_{1:k}])}(\theta\in S)  + \tilde{\delta}.\\
		P_{\theta\sim \cA([Z,z_{1:k}])}(\theta\in S)  \leq  e^{\epsilon(Z,z_1)+\epsilon([Z,z_1],z_2)+...+\epsilon([Z,z_{1:k-1}],z_k)} P_{\theta\sim \cA(Z)}(\theta\in S)   + \tilde{\delta}.
		\end{align*}
		for $\tilde{\delta} = \sum_{i=1:k} \left[\delta([Z,z_{1:i-1}],z_i)\prod_{j=1:i-1}e^{\epsilon([Z,z_{1:j-1}],z_j)}\right]$.
	\end{enumerate}
\end{fact}
\begin{proof}
	These properties all directly follow from the proof of these properties for differential privacy (see e.g., \citep{dwork2014algorithmic}), as the uniformity over data sets is never used in the proof. The group privacy is more involved	since the size of the data set changes as the size of the privacy target (now a fixed group of people) gets larger, group privacy statement follows from a simple calculation that repeatedly applies the definition of pDP for a different data set. 
\end{proof}

	\subsection{The distribution and moments of pDP}\label{sec:distribution_of_pDP}
One useful notion to consider in practice is to understand exactly how much privacy loss is incurred for those who participated in the data set. This is practically relevant, because if a cautious individual decides to not submit his/her data, he/she would necessarily do it by rejecting a data-usage agreement and therefore the data collector is not legally obligated to protect this person and in fact does not have access to his/her data in the first place.
 
 {
It is debatable whether it is as important to protect individuals who are not in the data set as those who are. We will illustrate this point with an example.
Suppose the Federal government is to decide on a potential funding support based on whether a township's average household income qualifies for it.  Household income is clearly considered sensitive information, so the census bureau decides to add a Laplace noise to the average income to prevent privacy risk. The noise level is chosen independently of the data such that it will not make the funding decision impossible. If we consider the richest person on earth, it is possible that his/her household income is larger than the total GDP of this township. The noise-adding algorithm does not provide a meaningful DP guarantee to him/her, as it will be straightforward to infer with high confidence that this person does not live in this township. But does it matter? It is unlikely that the richest person on earth would consider this kind of inference about him/her a breach of privacy.
}

 {
pDP provides analytical tools to formally study the privacy of only those people in a data set. 
It offers a natural way to analyze and also empirically estimate any statistics of the pDP losses  over a distribution of data points corresponding to a fixed randomized algorithm $\cA$.
}
\begin{definition}[Moment pDP for a distribution]
	Let $(Z,z)$ be drawn from some distribution (not necessarily a product distribution) $\cP$, it induces a distribution of $\epsilon(Z,z)$. Then we say that the distribution obeys $k$th moment per-instance DP with parameter vector 
	$
	(\E \epsilon, \E[\epsilon^2] , ...,\E[\epsilon^k],\delta).
	$
\end{definition}

The moments of pDP and the corresponding view of pDP's privacy loss as a random variable is a powerful idea and it enables flexible and comprehensive descriptions of the privacy \emph{footprint} of a randomized algorithm on a set of targets subject to a constraint or a distribution of the input data set.  We give a few examples below.

\begin{description}
	\item[pDP of a data set.] Let $Z$ be a fixed data set, when we choose $\cP$ to be a discrete uniform distribution supported on $\{(Z_{-i}, z_i)\}_{i=1}^n$ with probability $1/n$ for each $i$. Then taking $k=2$ allows us to calculate the mean and variance of the privacy loss of individuals in a data set, and taking higher order $k$ allows us to produce quantile estimates and  high probability tail bounds of the random-variable of an average user in the data set. 
	\item[pDP for all.] When we fix $Z$ and but allow $z$ to be drawn any distribution defined on $\cZ$, then this becomes a much stronger notion of privacy that protects all individual $z\in\cZ$ provided that the data set is $Z$. This is closely related to local sensitivity \citep{nissim2007smooth,dwork2009differential}.
	\item[pDP for one.] When we fix $z$ but allow $Z$ to be drawn from any distributions defined on a collection of data sets, then this becomes the worst case privacy loss that can happen to a given individual $z$. This notion is closely related to the personalized DP \citep{ghosh2015selling,ebadi2015differential,liu2015fast} and could be useful when individuals have different sensitivity.
		\item[pDP with assumptions on data sets.] As we mentioned in the introduction, a branch of modern machine learning focuses on finding reasonable assumptions on the data sets which reduces the computation and sample complexity of a problem. In this case, $Z$ could be drawn from any distributions such that these assumptions are true with probability $1$, in other words, we can take advantage of these assumptions when calculating the privacy loss of an individual $z$.
	\item[pDP with a data set prior.] When we take $Z\sim \pi$ for some prior distribution $\pi$, then the moments of pDP makes it possible to take advantage of that prior distribution to describe the privacy of an individual $z$ as a distribution over the possible privacy loss. 
\end{description}

As an illustration, we compare pDP of a data set, pDP for all and the classical differential privacy using a simulated experiment. The results are shown in Figure~\ref{fig:pDP_on_data}. As we can see, the more fine-grained per-instance DP reveals more than an order of magnitude stronger privacy protection for all users in the data set, and six times better privacy protection for all users in the entire universe, than the standard DP's characterization.  We will revisit some of these notions in our case study for linear regressions in Section~\ref{sec:casestudy} with concrete bounds. 

\begin{figure}[tb]
	\centering
	\includegraphics[width=0.8\linewidth]{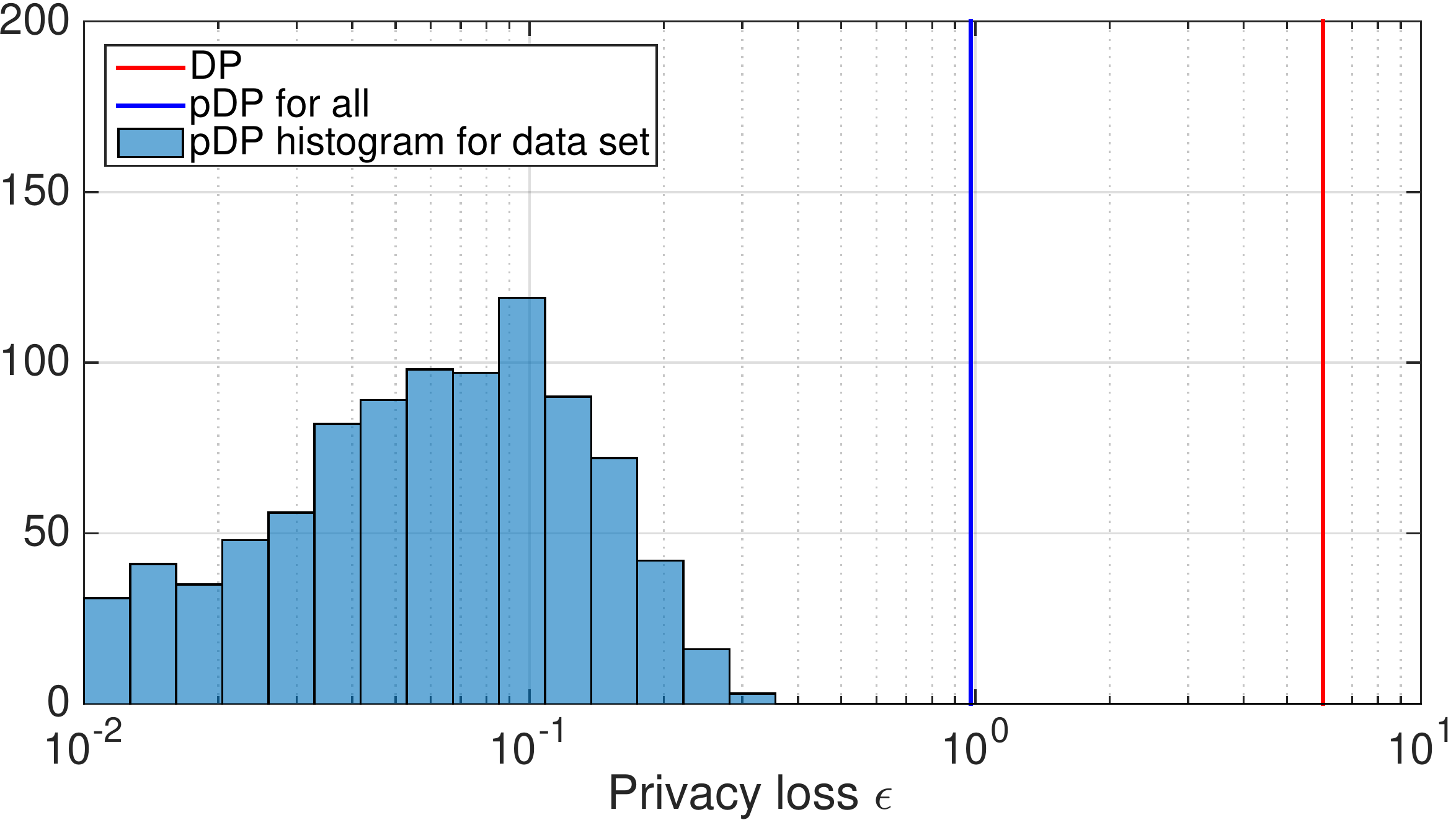}
	\caption{Illustration of the privacy loss $\epsilon$ of an output perturbation algorithm under DP, pDP for all, as well as the distribution of pDP's privacy loss for data points in the data set. The data set is generated by a linear Gaussian model, where the design matrix is normalized such that each row has Euclidean norm $1$ and $y$ is also clipped at $[-1,1]$. The output perturbation algorithm releases $\hat{\theta} \sim \cN((X^TX + I)^{-1}X\mathbf{y}, \sigma^2 I)$ with $\sigma=4$. Our choice of $\delta=10^{-6}$.}
	\label{fig:pDP_on_data}
\end{figure}

\subsection{Generalization and domain adaptation}

Assume that the data set is drawn iid from some unknown distribution $\cD$ --- a central assumption in statistical learning theory --- then we can take $\cP=\cD^{n-1}\times \cD$. This allows us to use the moment of pDP losses to capture on average how well data points drawn from $\cD$ are protected. It also controls generalization error, and more generally cross-domain generalization. 
\begin{definition}[On-average generalization]
	Under the standard notations of statistical learning, the on-average generalization error of an algorithm $\cA$ is defined as
	$$
			\text{Gen}(\cA,\cD,n) = \left|\E_{Z\sim \cD^n, z\sim \cD} \E_{\theta\sim \cA(Z)} \frac{1}{n}\sum_{i=1}^n\ell(\theta,z_i) - \ell(\theta,z)\right|
	$$
\end{definition}
\begin{proposition}[Moment pDP implies generalization]\label{prop:gen}
	Assume bounded loss function $0\leq \ell(\theta,z) \leq 1$. Then
	the on-average generalization is smaller than 
			$$\E_{Z\sim \cD^n} (\E_{z\sim \cD}[e^{\epsilon(Z,z)}|Z])^2-1 + \E_{Z\sim \cD^n,z\sim\cD} \delta(Z,z)
			+(\E_{Z\sim \cD^n}\E_{z\sim\cD}[e^{\epsilon(Z,z)}|Z]\E_{z\sim\cD}[\delta(Z,z)|Z].$$
		\end{proposition}
		
		Note that this can also be used to capture the privacy and generalization of transfer learning (also known as domain adaptation) with a fixed data set or a fixed distribution. Let the training distribution be $\cD$ and target distribution be $\cD'$, 
		
		Take $\cP = \cD^n \otimes \cD'$  or $\cP = \delta_{Z} \otimes \cD'$.
		In practice, this allows us to upper bound the generalization to the Asian demographics group, when the training data is drawn from a distribution that is dominated by white males (e.g., the current DNA sequencing data set). 
		We formalize this idea as follows.
		\begin{definition}[Cross-domain generalization] Assume $0 \leq \ell(\theta,z)\leq 1$.
			The on-average cross-domain generalization with base distribution $\cD$ to target distribution $\cD'$ is defined as:
			$$
			\text{Gen}(\cA,\cD,\cD',n) \leq \left|\E_{Z\sim \cD^n, z\sim \cD'} \E_{\theta\sim \cA(Z)} \left[\frac{1}{n}\sum_{i=1}^n\rho_i\ell(\theta,z_i) - \ell(\theta,z) \right]\right|.
			$$
			where $\rho_i = \cD'(z_i)/\cD(z_i)$ is the inverse propensity (or importance weight) to account for the differences in the two domains.
		\end{definition}
		\begin{proposition}\label{prop:cross-gen}
			The cross-domain on-average generalization can be bounded as follows:
			$$
			\text{Gen}(\cA,\cD,\cD',n)=\E_{Z\sim \cD^{n-1},\{z'\}\sim\cD,z''\sim\cD'}[(e^{\epsilon(Z,z') + \epsilon(Z,z'')} - 1) + \delta(Z,z') + \epsilon(Z,z')\delta(Z,z'')]
			$$
\end{proposition}

		The expressions in Proposition~\ref{prop:gen}~and~\ref{prop:cross-gen} are a little complex, we will simplify them to make it more readable.
		\begin{corollary}\label{cor:gen-cross}
			Let $\sup_{Z,z}\delta(Z,z)\leq \delta$, and $\E_\cD[e^{2\epsilon(Z,z)}] \leq 1$ and for simplicity, we write $\E_{Z\sim \cD^n, z\sim \cD} \epsilon(Z,z) = \E_\cD f$ and $\E_{Z\sim \cD^n, z\sim \cD'} \epsilon(Z,z) = \E_{\cD'} f$. Then the cross-domain on-average generalization is smaller than
			$$
			\frac{1}{2}[\E_\cD e^{2\epsilon}+\E_{\cD'} e^{2\epsilon}] -1  + 2\delta = \frac{1}{2}\left[\sum_{i=1}^\infty \frac{2^i}{i!}\E_\cD \epsilon^i + \E_{\cD'}\epsilon^i\right]  + 2\delta.
			$$
		\end{corollary}
		
	It will be interesting to compare the quantity to R\'{e}nyi-DP which also uses the moment generating function of the privacy random variable. The difference is that in R\'{e}nyi-DP, the privacy random variable is induced by the distribution of the output, while here it is induced by the distribution of the data set and the target.

\section{Per-instance sensitivity in smooth learning problems}
In this section, we present our main results and give concrete examples in which per-instance sensitivity (hence per-instance privacy) can be analytically calculated. Specifically, we consider following regularized empirical risk minimization form:
\begin{equation}\label{eq:erm}
\hat{\theta}  = \argmin_{\theta} \sum_i \ell(\theta,z_i)  + r(\theta),
\end{equation}
or in the non-convex case, finding a local minimum.
$\ell(\theta,z)$ and $r(\theta)$ are the loss functions and regularization terms. We make the following assumptions:
\begin{enumerate}
	\item[A.1.]  $\ell$ and $r$ are differentiable in argument $\theta$.
	\item[A.2.] The partial derivatives are absolute continuous, i.e., they are twice differentiable almost everywhere and the second order partial derivatives are Lebesgue integrable.
\end{enumerate}
Our results under these assumptions will cover learning problems such as linear and kernel machines as well as some neural network formulations (e.g., multilayer perceptron and convolutional net with sigmoid/tanh activation), but not non-smooth problems like lasso, $\ell_1$-SVM or neural networks ReLU activation. We also note that these conditions are implied by standard assumptions of strong smoothness (gradient Lipschitz) and do not require the function to be twice differentiable everywhere. For instance, the results will cover the case when either $\ell$ or $r$ is a Huber function, which is not twice differentiable. 

Technically, these assumptions allow us to take Taylor expansion and have an integral form of the remainder, which allows us to prove the following stability bound.
\begin{lemma}\label{lem:smoothlearning}
	Assume $\ell$ and $r$ satisfy Assumption A.1 and A.2. 
	Let $\hat{\theta}$ be a stationary point of $\sum_i \ell(\theta,z_i)  + r(\theta)$, $\hat{\theta}'$ be a stationary point $\sum_i \ell(\theta,z_i) + \ell(\theta,z) + r(\theta)$ and in addition, let $\eta_t = t\hat{\theta} +(1-t)\hat{\theta}'$ denotes the interpolation of $\hat{\theta}$ and $\hat{\theta}'$.
	Then the following identity holds:
		\begin{align*}
		\hat{\theta}-\hat{\theta}' &= \left[\int_0^1\left(\sum_i\nabla^2\ell(\eta_t,z_i) +\nabla^2\ell(\eta_t,z) + \nabla^2r(\eta_t)\right) dt \right]^{-1} \nabla \ell(\hat{\theta},z)\\
		&= -\left[\int_0^1\left(\sum_i\nabla^2\ell(\eta_t,z_i) + \nabla^2r(\eta_t)\right) dt \right]^{-1} \nabla \ell(\hat{\theta}',z).
		\end{align*}
\end{lemma}
The proof uses first order stationarity condition of the optimal solution and apply Taylor's theorem on the gradient. 

	The lemma is very interpretable. It says that the perturbation of adding or removing a data point can be viewed as a one-step quasi-newton update to the parameter. Also note that $\nabla \ell(\hat{\theta}',z)$ is the ``score function'' in parametric statistical models, and when $\ell(\hat{\theta}',(x,y)) =  \ell(f_{\hat{\theta}}(x),y)$,  it is  the product of the ``pseudo-residual'' $\frac{\partial \ell}{\partial f}$ and the gradient direction $\nabla f$ in gradient boosting \citep[see e.g.,][Chapter 10]{friedman2001elements}. 

The result implies that the per-instance sensitivity in $\|\cdot\|_A$ for some p.d. matrix $A$ can be stated in terms of a certain norm of the ``score function'' specified by a quadratic form $H^{-1}AH^{-1}$, and therefore by Lemma~\ref{lem:multigaussianDP}, the output perturbation algorithm:
\begin{equation}\label{eq:gaussian_mech}
\ttheta \sim  \cN(\hat{\theta}(X), A^{-1}/\gamma),
\end{equation}
obeys $(\epsilon,\delta)$-pDP for any $\delta>0$ and
\begin{equation}\label{eq:privacyloss}
\epsilon(Z,z)  = \sqrt{\nabla\ell(\hat{\theta}',z)^T H^{-1}AH^{-1} \nabla\ell(\hat{\theta}',z) \log(1.25/\delta)}.
\end{equation}
This is interesting because for most loss functions the ``score function'' is often proportional to the prediction error of the fitted model $\hat{\theta}'$ on data point $z$ and this result suggests that the more accurately a model predicts a data point, the more private this data point is. This connection is made more explicit when we specialize to linear regression and the per-instance sensitivity
\begin{equation}\label{eq:sensitivity_linreg}
			\Delta_A(Z,z) = |y-x^T\hat{\theta}|\sqrt{ x^T([X']^TX')^{-1}A([X']^TX')^{-1}x}= |y-x^T\hat{\theta}'|\sqrt{ x^T(X^TX)^{-1}A(X^TX)^{-1}x }.
\end{equation}
is clearly proportional to prediction error. In addition, when we choose $A\approx X^TX$, the second term becomes either $\mu :=  x^T([X,x]^T[X,x])^{-1}x$ or $\mu' := x^T(X^TX)^{-1}x$, which are ``in-sample'' and ``out-of-sample'' \emph{statistical leverage scores} of $x$. 

Leverage score measures the importance/uniqueness of a data point relative to the rest of the data set and it is used extensively in regression analysis \citep{chatterjee1986influential} (for outlier detection and experiment design), graph sparsification (for adaptive sampling) \citep{spielman2011graph} and numerical linear algebra (for fast matrix computation)\citep{drineas2012fast}. To the best of our knowledge, this is the first time leverage scores are shown to be connected to differential privacy.

\section{Case study: pDP analysis in linear regression}\label{sec:casestudy}

So far we have described output perturbation algorithms with a fixed noise adding procedure. However in practice it is not known ahead of time how to choose $A$. 
Assume all $x$ are normalized to $\|x\|=1$\footnote{{ This assumption simplifies the presentation. In general, we can conduct pDP analysis to any randomized algorithm on any data set. However, as \citet{cummings2015truthful} has shown, if the data domain is not bounded, then no algorithm can differentially privately release linear regression coefficients with non-trivial accuracy in general.}}, denote $\mu_2(x) := x^T(X^TX)^{-2}x$, $\mu_1(x) := x^T(X^TX)^{-1}x$. We discuss the pros and cons of the three natural choices.

\begin{itemize}
	\item $A\approx \lambda_{\min} I$: This corresponds to the standard $\ell_2$-sensitivity and it adds an isotropic noise and provides a uniform guarantee for all data-target pairs where $X^TX$ has smallest eigenvalue $\lambda_{\min}$, because $\sup_{x}\sqrt{\mu_2(x)} \leq 1/\lambda_{\min}$, but it adds more noise than necessary for those with much smaller $\mu_2(x)$.
	\item $A \approx (X^TX)^2$:  We call this the ``democratic'' choice conditioned on the data set, as it homogenizes the ``leverage'' part of the per-instance sensitivity of points to $\|x\|=1$ so any $x$ gets about the same level of privacy. It, however, is not robust to if our data-independent choice of $A$ is in fact far away from the actual $(X^TX)^2$.
	\item $A \approx X^TX$: We call this the ``Fisher''-choice, because the covariance matrix will be proportional to the inverse Fisher information, which is the natural estimation error of $\hat{\theta}$ under the linear regression assumption. The advantage of this choice is that conducting statistical inference, e.g., t-test and ANOVA for linear regression coefficients would be trivial. 
\end{itemize}
Interestingly, for linear and ridge regression, the second and third choices are closely related to popular algorithms studied before. In fact, taking $A = (X^TX)^2$ recovers the objective perturbation (\OBP{}) method\citep{chaudhuri2011differentially, kifer2012private}:
\begin{equation}\label{eq:OBP}
\hat{\theta} =  \argmin_{\theta\in \Theta}  \|\mathbf{y} - X \theta\|^2 +  \langle z, \theta  \rangle,  \quad \theta\sim \cN(0,\sigma^2 I). 
\end{equation}
while taking $A = X^TX$ recovers the one-posterior-sampling (OPS) mechanism proposed in \citep{dimitrakakis2014robust,wang2015privacy}, which outputs
\begin{equation}\label{eq:OPS_flatprior}
\hat{\theta}  \sim P(\theta |  X,\mathbf{y}) \propto  e^{- \gamma \|\mathbf{y} - X \theta\|^2}.
\end{equation}
 An important difference is that in \OBP{} and \OPS{}, $A$ is not fixed, but rather depends on the data. As a result, we cannot use Lemma~\ref{lem:multigaussianDP} to calculate the pDP. In fact, if the data-target can be arbitrary and $r=0$, the data-independent choice of $A$ could imply an unbounded $\epsilon$ (consider an arbitrarily near singular $X$ and $x$ in its null space).

Not surprisingly, existing analyses of \OBP{} and \OPS{} require additional assumptions.  \citet{kifer2012private} adds an additional $\lambda\|\theta\|^2$ to \eqref{eq:OBP}, while \citet{wang2015privacy} assumes that the loss function is bounded (by modifying it or constraining the domain $\Theta$) so that the exponential mechanism \citep{mcsherry2007mechanism} would apply. It was later pointed out in \citep{foulds2016theory} that OPS is not asymptotically efficient in that it has an asymptotic relative efficiency (ARE) inversely proportional to $\epsilon$, while simple sufficient statistics perturbation can achieve asymptotic efficiency comparable to \citep{smith2008efficient}.

In the remainder of the section, we will first zoom into the \OPS{}  and propose a direct analysis of pDP using Lemma~\ref{lem:smoothlearning}, then we will describe how to use the pDP analysis to obtain an extension of \OPS{} that obeys $(\epsilon,\delta)$-DP  and asymptotically efficient under the same data assumption in \citep{foulds2016theory}. We will see that \OPS{} effectively converges to the ``Fisher''-choice of noise adding in the same asymptotic regime and offers dimension and condition number independent expected pDP loss.

\subsection{pDP analysis of \OPS{}}


The first result calculates the pDP loss of OPS.
\begin{theorem}[The adaptivity of OPS in Linear/Ridge Regression]\label{thm:OPS}
	Consider the \OPS{} algorithm that samples from 
	$$
	p(\theta|X,\mathbf{y}) \propto e^{-\frac{\gamma}{2} \left(\|\mathbf{y} - X\theta\|^2 + \lambda\|\theta\|^2\right)}.
	$$
	Let $\hat{\theta}$ and $\hat{\theta}'$ be the ridge regression estimate with data set $X\times \mathbf{y}$ and $[X,x]\times [\mathbf{y},y]$ and defined the out of sample leverage score
	$\mu := x^T(X^TX + \lambda I)^{-1}x = x^TH^{-1}x$ and in-sample leverage score $\mu' := x^T [(X')^TX' + \lambda I]^{-1}x = x^T(H')^{-1}x $.
	Then for every $\delta>0$, privacy target $(x,y)$, the algorithm is $(\epsilon,\delta)$-pDP with
	\begin{align}
		\epsilon(Z,z) \leq& 	\frac{1}{2}\left| -\log(1+\mu) + \frac{\gamma\mu}{(1+\mu)}(y-x^T\hat{\theta})^2\right| + \frac{\mu}{2} \log (2/\delta) + \sqrt{\gamma\mu \log (2/\delta)} |y-x^T\hat{\theta}| \label{eq:eps_master1}\\
		=&\frac{1}{2}\left|-\log(1-\mu') - \frac{\gamma\mu'}{1-\mu'}(y-x^T\hat{\theta}')^2\right| + \frac{\mu'}{2} \log(2/\delta) + \sqrt{\gamma\mu'\log(2/\delta)}|y-x^T\hat{\theta}'|.\label{eq:eps_master2}
	\end{align}
\end{theorem}
The proof is given in the appendix. 

The two equivalent upper bounds are both useful. \eqref{eq:eps_master1} is ideal for calculating pDP when $x$ is not in the data set and \eqref{eq:eps_master2} is perfect for the case when $x$ is in the data set.

\begin{remark}
		The bound \eqref{eq:eps_master1} can be simplified to
		$$
		\frac{\mu}{2}(1+\log(2/\delta)) + \frac{1}{2}\gamma\min(\mu,1)|y-x^T\hat{\theta}'|^2 + \sqrt{\gamma \mu\log(2/\delta)}|y-x^T\hat{\theta}'|.
		$$
	If $\mu=o(\log(2/\delta))$\footnote{ This is not an unrealistic assumption because 
		$\mu$ and $\mu'$ are $o(1)$ as long as $x$ is bounded and the minimum eigenvalue of $X^TX+\lambda I$ is $\omega(1)$. This is required for (agnostic) linear regression to be consistent and is implied by the condition that the  population covariance matrix $\frac{1}{n}\E X^TX$ is full rank. } and we choose $\gamma$ such that $\sqrt{\gamma\mu'\log(2/\delta)}|y-x^T\hat{\theta}'| \leq 1$, then the bound can be simplified to 
	$$\epsilon(Z,z) \leq  2\sqrt{\gamma\mu\log (2/\delta)}|y-x^T\hat{\theta}| + o(1).$$
	This matches the order of Gaussian mechanism with a fixed (data-independent) covariance matrix.
\end{remark}

The results in \citep{foulds2016theory} are stated for general exponential family models under a set of assumptions that translate into the following for linear regression:
\begin{enumerate}
	\item[(a)] data $x_1,...,x_n$ is drawn i.i.d. from $\cD$ supported on $\cX$ where $\cX\subset \cB_{\|\cdot\|_2}(1)$.
	\item[(b)] population covariance matrix $\frac{m}{d} I \preceq \E_\cD xx^T \preceq \frac{M}{d} I$ for constant $m$ and $M$, 
	\item[(c)] $y_i \sim \cN(x_i^T\theta_0,\sigma^2)$ for some $\theta_0$. 
\end{enumerate}
To simplify the presentation, we also assume $n$ scales with respect to $d$ such that
\begin{enumerate}
	\item[(d)] with high probability, $XX^T \succ \frac{\alpha n}{2d}I$.
\end{enumerate}
The last assumption measures how quickly the empirical covariance matrix $\frac{1}{n}XX^T$ concentrates to $\E_{x\sim\cD} xx^T$. It can be shown that if $X$ is an appropriately scaled subgaussian random matrix, this happens with probability $1-n^{-10}$ whenever $n>\max(10d,10d^{-2/3}\log n)$.


\begin{proposition}\label{prop:OPS-performance}
	The sequence of OPS algorithm with parameter $\gamma_n$, $\lambda_n$ obeys the following properties.
	\begin{enumerate}
		\item \textbf{pDP and DP in the agnostic setting.}  Assume $\|x\|\leq 1$ for every $x\in\cX$. The algorithm obeys $(\epsilon_n,\delta)$-pDP, for each data set $(X,\mathbf{y})$ and all target $(x,y)$,
			\begin{equation}\label{eq:pDP_linreg_agnostic}
			\epsilon_n = \sqrt{\frac{\gamma_n\log(2/\delta)}{\lambda_n+\lambda_{\min}}}|y-x^T\hat{\theta}|+  \frac{\gamma_n|y-x^T\hat{\theta}|^2}{2\max\{\lambda_n + \lambda_{\min},1\}}  + \frac{\gamma_n(1+\log(2/\delta))}{2(\lambda_n+\lambda_{\min})}.
			\end{equation}
		If we further assume $|y|<1$, then $\sup_{(X,\mathbf{y}),(x,y)}|y-x^T\hat{\theta}|= 1+n^{1/2}\lambda_n^{-1/2}$ and the algorithm obeys $(\epsilon_n,\delta)$-DP with
			\begin{equation}\label{eq:DP_linreg_agnostic}
						\epsilon_n =\sqrt{\frac{2(n+\lambda_n)\gamma_n\log(2/\delta)}{\lambda_n^2}}+  \frac{2(n+\lambda_n) \gamma_n}{\lambda_n\max\{1,\lambda_{n}\}}  + \frac{\gamma_n(1+\log(2/\delta))}{2\lambda_n}.
			\end{equation}
		\item \textbf{pDP under model assumption.} Assume conditions (a)(b)(c)(d) above are true, and also $\gamma_n = \omega(1)$, $\lambda_n=o(\sqrt{n})$. Then with high probability over the joint distribution of $(X,\mathbf{y})$, the algorithm with $\gamma_n \leq \frac{4n\log(2/\delta)}{\max\{d,(1+\log(2/\delta))^2\}}$
		obeys 
		$\left(\epsilon_n,\delta\right)$-pDP with
		\begin{equation*}
		\epsilon_n = 
		\begin{cases}
			O\left(\sqrt{\frac{(1+\|\theta_0\|)^2d\gamma_n}{\alpha n}\log(\frac{2}{\delta})}\right) &\text{ for all (x,y) satisfying $\|x\|=O(1)$ and $y = O(1)$.}\\
			O\left(\sqrt{\frac{\sigma^2d\gamma_n}{\alpha n}\log(\frac{2}{\delta})\log(\frac{2}{\delta'})}\right) &\text{ for any $x\in\cX$ with probability $1-\delta'$ over $y\sim \cN(\theta_0^Tx,\sigma^2)$.}
		\end{cases}
		\end{equation*}
		Moreover, with probability $1-n\delta'$ over the conditional distribution $\mathbf{y}|X$, 
		the privacy loss of $(x_1,y_1),...,(x_n,y_n)$ obeys
		$$
			\frac{1}{n}\sum_{i=1}^n\epsilon_n\big((X,\mathbf{y}), (x_i,y_i)\big)^2 = O\Big(\frac{ \sigma^2d\gamma_n}{n}\log(2/\delta)\log(2/\delta')\Big),
		$$
		which does not depend on $\alpha$ --- the smallest eigenvalue of $dX^TX/n$.
		
		
			\item \textbf{Statistical efficiency.} for every realization of data set $X$ such that $n>d$ and let the smallest eigenvalue of $X^TX$ be $\lambda_{\min}$, then
						$$\E_{\mathbf{y}\sim \cN(X\theta_0, \sigma^2 I_n)}\left[\|\tilde{\theta} - \theta_0\|^2 \middle| X \right] = \sigma^2\tr[(X^TX + \lambda_n I)^{-1}](1+\gamma_n^{-1}) + \lambda_n^2\|(X^TX + \lambda_n I)^{-1}\theta_0\|^2$$
			If $\lambda_{\min} = \Omega(d/n)$ ( this is true with high probability under assumption (b)(d))
		Then we get
			$$\E_{\mathbf{y}\sim \cN(X\theta_0, \sigma^2 I_n)}\left[\|\tilde{\theta} - \theta_0\|^2  \middle| X \right] = \sigma^2\tr[(X^TX + \lambda_n I)^{-1}](1+\gamma_n^{-1}) + O(\frac{\lambda_n^2d^2\|\theta_0\|^2}{n^2})$$
		In other word, the estimator is asymptotically efficient, for all $\lambda_n = o(n^{1/2})$ and $\gamma_n = \omega(1)$.
			\item \textbf{Optimization error.} Let $F(\theta) = 0.5\|\mathbf{y}- X\theta\|^2 + \lambda \|\theta\|^2$ and $\hat{\theta}=\argmin F(\theta)$, then
		$$
		\E F(\tilde{\theta}) - F(\hat{\theta})  =  d/\gamma_n,
		$$
		and also with probability at least $1-\delta$ over $P(\tilde{\theta}|Z)$
		$$
		F(\tilde{\theta}) - F(\hat{\theta}) \leq d\log(d/\delta)/\gamma_n.
		$$
		With $\gamma_n  =  \frac{\epsilon^2\alpha n}{d\log(2/\delta)}$, the result matches the information-theoretic lower bound for differentially private empirical risk minimization \citep{bassily2014private}\footnote{Note that in \citep{bassily2014private}, the strong convexity parameter $\Delta$ is assumed for each loss function, therefore it maps into our $\alpha$ as $\Delta \asymp \alpha/d$. Also, we used that in \citep{bassily2014private}'s setting the loss function is Lipschitz within the bounded domain, which ensures $|y-x^T\theta| = O(1)$. }.
	\end{enumerate}
\end{proposition}
	We now discuss a few aspects of the above results.

\paragraph{pDP vs DP in the agnostic setting.}
Firstly, it highlights the key advantage of pDP over DP.  DP is not able to take advantage of desirable structures in the data set, while pDP provides a principled framework to handle them.

In particular, let us compare the pDP and DP in the agnostic setting, for the OPS that uses the same randomization. DP measures something that is completely data independent and corresponds specifically to a contrivedly constructed data set $(X,\mathbf{y})$ such that $\mathbf{y}$ is an eigenvector of $XX^T$ corresponding to a specific eigenvalue of magnitude $\sqrt{\lambda_n}$, this makes $\|\hat{\theta}\|_2$ as large as $\sqrt{n}/\sqrt{\lambda_n}$. Moreover, a target data point is chosen so that $x$ match the direction of $\hat{\theta}$. While this is a legitimate construction in theory, but it does not directly correspond to the specific data set that a statistician just spent two years collecting, and it is unreasonable that he/she will have to calibrate the amount of noise to inject to provide more reasonable protection to a pathological case that has nothing to do with the reality.

		pDP, on the other hand, makes it possible for the statistician to condition on the data set. If the statistician finds out that $\|\hat{\theta}\|_2=O(1)$, then the pDP loss is as small as $\sqrt{\gamma_n\log(2/\delta)/\lambda_{n}}$ for \emph{everyone} in the population. With $\gamma_n=n^{\alpha/2}$ and $\lambda_n = n^{1/2-\alpha/2}$ for any $\alpha>0$, the algorithm remains to be statistically efficient with an ARE of $(1+n^{-\alpha})$ yet can provide a strong privacy guarantee of $\epsilon_n = n^{-1/4+\alpha/2}$. 
If in addition, the statistician realized that the data set is \emph{well-conditioned}, that is, the maximum and minimum eigenvalue of $X^TX$ are on the same order of $n/d$, then we can further improve the bound by replacing $\lambda_n$ with $\lambda_{\min}+\lambda_n$. The statistician can happily get away with the same privacy guarantee ($\epsilon_n=n^{-1/4}$) while not having to add too much noise or even regularize at all (setting $\gamma_n=n^{1/2}$ and $\lambda_n=0$). Note that the condition number is a desirable property that governs how reliably one can hope to estimate the linear regression coefficients using the given data set.

We would like to emphasize that the pDP guarantee in the two cases we discussed above applies to everyone in the population $\{(x,y)|\|x\|\leq1,|y|\leq 1\}$, therefore such $(\epsilon,\delta)$-pDP guarantee is as powerful as $(\epsilon,\delta)$-DP after the data set is collected.

\paragraph{pDP-for-all vs average pDP on the data set.}
Secondly, unlike DP which always provides a crude upper bound for everyone, pDP is able to reflect the differences in the protection of different target person. 
	Under the model assumption, the average privacy loss of people in the data set, is scale-invariant and interestingly, also independent of the condition number (smallest eigenvalue). It is a factor of $(1+|\theta_0|)^2/m$ times smaller than the pDP guarantee for everyone in the population. This is significant for finite sample performance since $(1+\|\theta_0\|)/m$ (although they do not change with $n$), can be quite large.

\paragraph{pDP under covariate shift.}
Lastly,  if we consider a setting in between the above two, where the target $x$ can be drawn from any distribution defined on $\cX$ that could be arbitrarily different from the training data distribution, then the scale-invariant property remains (the factor of $(1+\|\theta_0\|)$ is dropped). This is relevant in causal learning when the $\E(y|x)$ is specified by some physical principles that are invariant to the distribution of $x$. In this case, the moments of the pDP would imply a much stronger notion of cross-domain generalization than what we show in Proposition~\ref{prop:cross-gen} since it does not depend on the target distribution of interest.

\paragraph{Improved DP guarantee for OPS.}
The proposition also improves the existing analysis for the OPS algorithm as a byproduct. The first statement shows that OPS preserves a meaningful (almost constant) differential privacy when $\gamma_n=1$ and $\lambda_n=\sqrt{n/d}$ without requiring a constant boundedness in the domain $\Theta$ or clipping the loss function like in \citet{wang2015privacy}. As a matter of fact, the ridge regression solution $\hat{\theta}$ could be in a ball of radius  $\Theta(n^{1/4})$, and even if we impose the smallest domain bound that covers $\hat{\theta}$, by exponential mechanism, the algorithm only obeys a pure $O(n^{1/2})$-DP, in contrast to the $(O(\log(1/\delta)),\delta)$-DP that we showed in the proposition above.

Despite the improvement, the DP guarantee is still a little unsatisfactory. If we require $(\epsilon,\delta)$-DP with constant $\epsilon$, then the OPS algorithm with $\lambda_n=\sqrt{n}$ is not asymptotically efficient (although it does achieve the optimal $O(1/n)$ rate).

Meanwhile, there are algorithms that attain asymptotic efficiency either by subsample-and-aggregate \citep{smith2008efficient} or by simply adding noise to the sufficient statistics \citep{dwork2010differential,foulds2016theory}.
So the question becomes: can we modify OPS such that it becomes asymptotically efficient with $(\epsilon_n,\delta)$-differentially private with $\epsilon_n=o(1)$?

We address this issue next.



		\subsection{``pDP to DP conversion'' and \AdaOPS{}}\label{sec:conversion}

In this section, we resolve the dilemma described earlier by using the idea of \citet{dwork2009differential}. 
The new algorithm, which we call \AdaOPS{}, adaptively and differentially privately chooses the tuning parameter $\lambda_n$ and $\gamma_n$ according to properties of the data set and privacy requirement. A pseudocode of \AdaOPS{} is given in Algorithm~\ref{alg:AdaOPS}. We acknowledge that the same idea of adaptively adding regularization term is not new and had been used by \citet{kifer2012private,blocki2012johnson,sheffet2017differentially} for analyzing other related differentially private algorithms. Our contribution here is only to assemble the ideas together into a working algorithm and illustrate how pDP analysis can help us design data-dependent DP algorithm that takes a  prescribed $(\epsilon,\delta)$ budget as an input.

\begin{algorithm} [h!]                   
	\caption{\AdaOPS{}: One-Posterior Sample estimator with adaptive regularization}          
	\label{alg:AdaOPS}                           
	\begin{algorithmic}                    
		\INPUT{ Data $X$, $\mathbf{y}$. Privacy target: $\epsilon$, $\delta$. And parameter $\kappa$ satisfying $0 \leq \kappa \leq \frac{n\epsilon}{4d(1+\log(4/\delta))}$ }
		\STATE{1. Calculate the minimum eigenvalue $\lambda_{\text{min}}(X^TX)$.}
		\STATE{2. Private release $\tilde{\lambda}_{\text{min}} =  \lambda_{\text{min}} + \frac{\sqrt{\log(4/\delta)}}{\epsilon/2}Z$, where $Z\sim \cN(0,1)$.}
		\STATE{3. Get one sample }
		$$\tilde{\theta}\sim \P(\theta|X,\mathbf{y}) \propto e^{-\frac{\gamma_n}{2} \left(\|\mathbf{y} - X\theta\|^2 + \lambda_n\|\theta\|^2\right)}$$
		with parameter 
		\begin{align*}
		\lambda_n &= \min\left\{0, \frac{n}{d\kappa} - \tilde{\lambda}_{\text{min}}+\frac{\log(4/\delta)}{\epsilon/2}\right\}\\
		\gamma_n & = \min\left\{\frac{n\epsilon^2}{16\kappa^2d^2\log(4/\delta)}, \frac{n\epsilon}{8\kappa^2d^2}\right\}
		\end{align*}
		\OUTPUT{ $\tilde{\theta}$}
	\end{algorithmic}
\end{algorithm}
The $\kappa$ parameter is the largest acceptable condition number in the data set. Often it can be determined independently of the data. It is used in the algorithm to rule out the pathological case of a possibly near-singular design matrix. We now analyze the properties of \AdaOPS{}.

\begin{proposition}\label{prop:AdaOPS}
	\begin{enumerate}
		\item Assume data domain is $\|x\|_2\leq 1$ and $|y|\leq 1$. The \AdaOPS{} estimator preserves $(\epsilon,\delta)$-DP.
		\item If Assumption (a)(b)(c) are true and in addition for the specific realization of $X$,  
		$$\lambda_{\min}(X^TX) > \frac{n}{\kappa d} + \frac{\sqrt{\log(10n)\log(4/\delta)}}{\epsilon/2},$$
		then, we have
		$$
		\E[\|\tilde{\theta} - \theta_0\|^2 | X]  = \left[1+\gamma_n\right]  \sigma^2\tr[(X^TX)^{-1}] + O(n^{-10})\|\theta_0\|^2.
		$$
		In other words, since $\gamma_n\leq \min\{\frac{\kappa^2d^2\log(4/\delta)}{n\epsilon^2}, \frac{\kappa^2d^2}{n\epsilon}\}$, the \AdaOPS{} estimator achieves asymptotic efficiency whenever $\epsilon$ obeys that
		$\min\{n\epsilon^2,s\epsilon\} = o(\kappa^2d^2\log(4/\delta)/n)$.
	\end{enumerate}
\end{proposition}

This proposition reveals that \AdaOPS{} improves over previous results in the literature \citep{smith2008efficient,foulds2016theory} in several ways. First of all, we only need $n\epsilon^2 = o(1)$ to achieve asymptotic efficiency. In contrast, \citep{foulds2016theory} does not provide non-asymptotic results with explicit dependence and \citep{smith2008efficient}'s bound for the subsample-and-aggregate method requires $n^{-1/5}\epsilon^{-6/5} = o(1)$ to achieve asymptotic efficiency.

Secondly, our bound has explicit dimension dependence while  \citep{foulds2016theory} and \citep{smith2008efficient} treat $d$ as a constant. In particular, our bound on the additive difference from exactly matching the Cramer-Rao lower bound of ($\sigma^2\tr((X^TX)^{-1})$) translates into $\sigma^2\tr[(X^TX)^{-1}]+ d^3/(n^2\epsilon^2)$. 

The extension from \OPS{} to \AdaOPS{}  is a good example of what we call ``pDP to DP conversion'', which follows the following procedures:
\begin{enumerate}
	\item Start with a fixed randomized algorithm of interest $\cA$. (e.g.,\OPS{})
	\item Calculate its pDP analytically (Proposition~\ref{thm:OPS}). 
	\item Inspect to identify key quantities ( in our case it is the strong convexity parameters).
	\item Differentially privately release high-probability confidence intervals of these key quantities (by releasing the smallest eigenvalue) and enforce the properties when needed (add regularization.)
\end{enumerate}
{ 
	``pDP to DP conversion'' uses the high-level idea of the Propose-Test-Release framework \citep{dwork2009differential}, which involves testing a sequence of conditions on key data-dependent quantities of the problem. Our approach is different because we propose to directly release these key quantities and intervene (regularize) if necessary. On the meta-level, a careful pDP analysis allows us to identify what these key quantities are that contributes to the sensitivity.
	Compared to the ``robust linear regression'' approach \citep[Section 4]{dwork2009differential}, \AdaOPS is avoids the need to discretize $\Theta$, hence does not require a runtime that is exponential in dimension $d$.
}

\subsection{Simulation}
\begin{figure}[tb]
	\centering
	\includegraphics[width=0.45\linewidth]{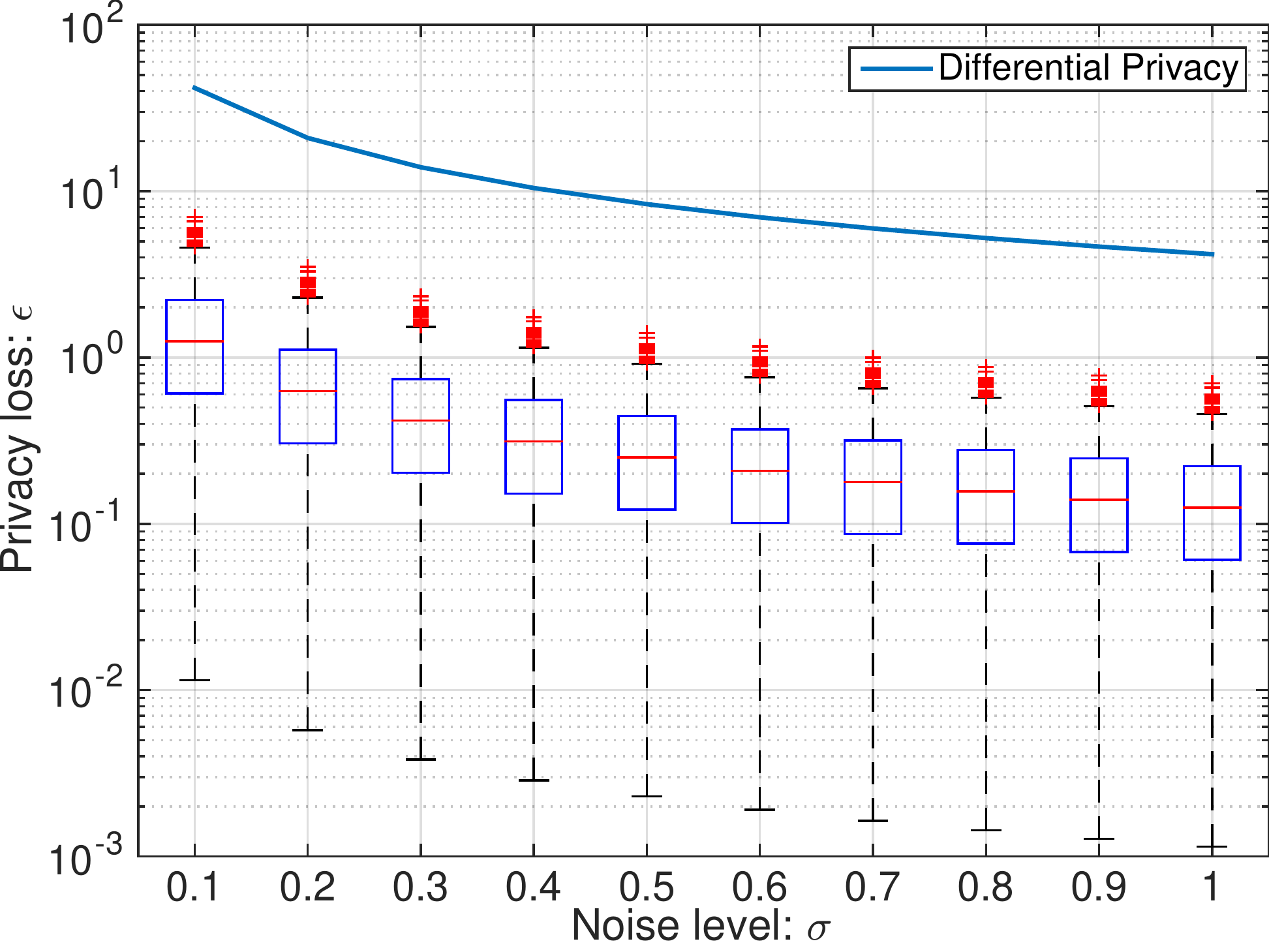}\hspace{1em}
	\includegraphics[width=0.45\linewidth]{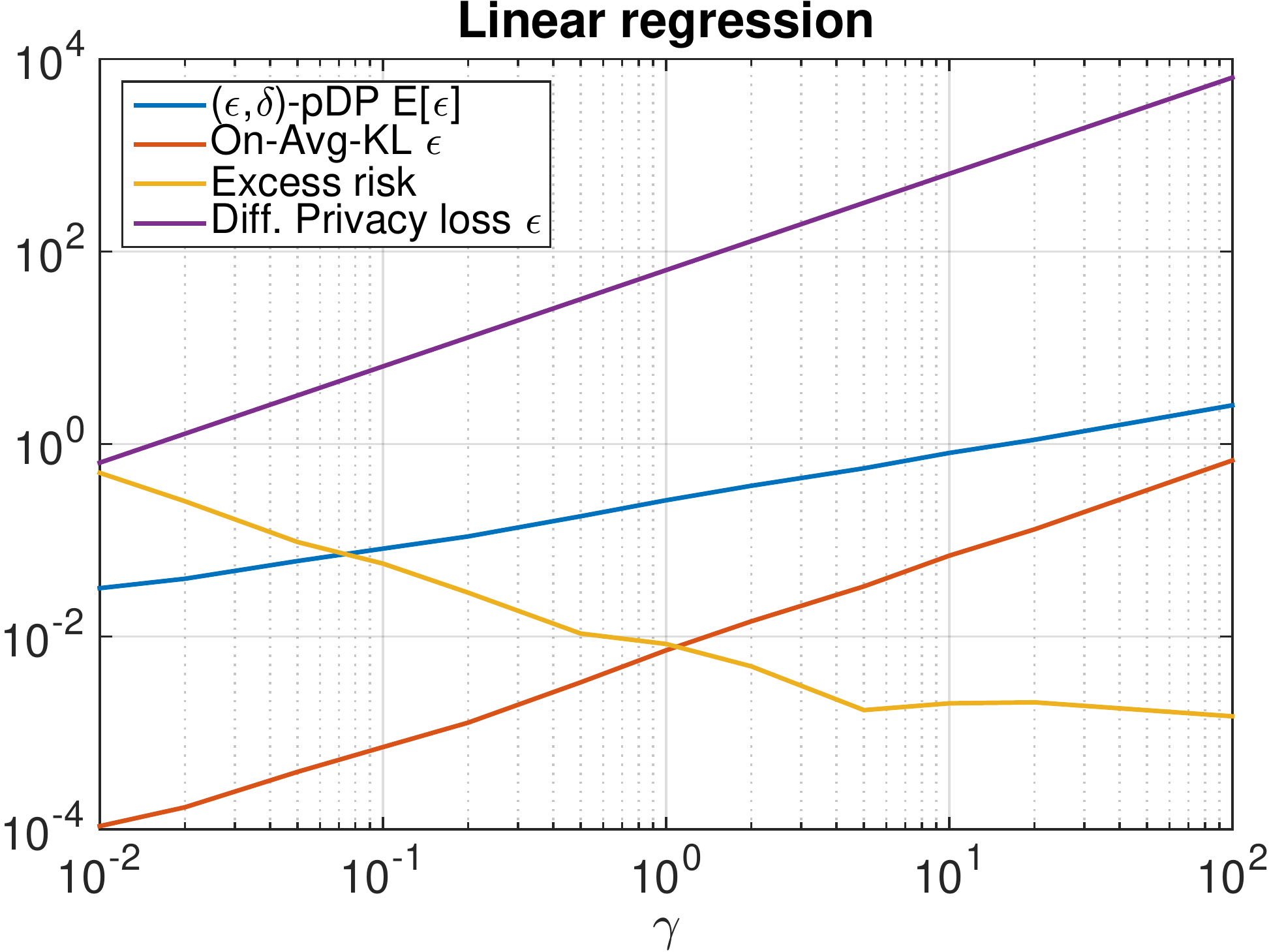}\\
	\caption{\textbf{Left:} $(\epsilon,\delta)$-DP and distribution of $(\epsilon(z,Z),\delta)$-pDP data points in linear regression with isotropic Gaussian noise adding. \textbf{Right:} Comparing the pDP privacy loss to the $\epsilon$-DP obtained through exponential mechanism \citep{wang2015privacy} using the same posterior sampling algorithm.
		In both experiment $\delta= 1e-6$. }
	\label{fig:pdpvsdp}
\end{figure}

We conclude the case study with two simulated experiments (shown in the two panes of Figure~\ref{fig:pdpvsdp}). In the first experiment, we consider the algorithm of adding isotropic Gaussian noise to linear regression coefficients and then compare the worst-case DP and the distribution of per-instance DP for points in the data set (illustrated as box plots). In the second experiment, we compare different notions of privacy to utility (measured as excess risk) of the fixed algorithm that samples from a scaled posterior distribution. In both cases, the average per-instance differential privacy over the data sets is several orders of magnitude smaller than the worst-case differential privacy.

\section{Concluding discussion}
In this paper, we proposed to use per-instance differential privacy (pDP) for quantifying the fine-grained privacy loss of a fixed individual against randomized data analysis conducted on a fixed data set. We analyzed its properties and showed that pDP is proportional to well-studied quantities, e.g., leverage scores, residual and pseudo-residual in statistics and statistical learning theory. This formalizes the intuitive idea that the more one can ``blend into the crowd'' like a chameleon, the more privacy one gets; and that the better a model fits the data, the easier it is to learn the model differentially privately. Moreover, the new notion allows us to conduct statistical learning and inference and take advantage of desirable structures of the data sets to gain orders-of-magnitude more favorable privacy guarantee than the worst case. This makes it highly practical in applications.

Specifically, we conducted a detailed case-study on linear regression to illustrate how pDP can be used. The pDP analysis allows us to identify and account for key properties of the data set, like the well-conditionedness of the feature matrix and the magnitude of the fitted coefficient vector, thereby provides strong uniform differential privacy coverage to everyone in the population whenever such structures exist. As a byproduct, the analysis also leads to an improved differential privacy guarantee for the \OPS{} algorithm \citep{dimitrakakis2014robust,wang2015privacy} and also a new algorithm called \AdaOPS{} that adaptively chooses the regularization parameters and improves the guarantee further. In particular, \AdaOPS{} achieves asymptotic statistical efficiency and differential privacy at the same time with stronger parameters than known before. 

The introduction of pDP also raises many open questions for future research. First of all, how do we tell individuals what their $\epsilon$s and $\delta$s of pDP are? This is tricky because the pDP loss itself is a function of the data, thus needs to be privatized against possible malicious dummy users. Secondly, the problem gets substantially more interesting when we start to consider the economics of private data collection. For instance, what happens if what we tell the individuals would affect their decision on whether they will participate in the data set? In fact, it is unclear how to provide an estimation of pDP in the first place if we are not sure what would the data be at the end of the day. Thirdly, from the data collector's point of view, the data is going to be ``easier'' and the model will have a better ``goodness-of-fit'' on the collected data, but that will be falsely so to some extent, due to the bias incurred during data collection according to pDP. How do we correct for such bias and estimate the real performance of a model on the population of interest? Addressing these problems thoroughly would require the joint effort of the community and we hope the exposition in this paper will encourage researchers to play with pDP in both theory and practical applications.


	\subsubsection*{Acknowledgments}
	The author thanks Steve Fienberg, Jing Lei, Ryan Tibshirani, Adam Smith, Jennifer Chayes and Christian Borgs for useful and inspiring discussions that motivate the work.
	We also thank the editor and anonymous reviewers for their helpful feedbacks that lead to significant improvements of the paper.
	\bibliographystyle{apa-good}
	\bibliography{FreeDP}

\appendix

\section{Proofs of technical results}
\begin{proof}[Proof of Proposition~\ref{prop:gen}]
	We first show that the first moment of pDP implies on-average stability and then on-average stability implies on-average generalization.
	
	Let $Z' = [Z,z']$, $Z'' = [Z,z'']$ and fix $z$. We first prove stability. Let $S = \theta|p(\theta)\geq p'(\theta)$
	\begin{align*}
	&\left|\E_{\theta \sim \cA(Z')} \ell(\theta,z) - \E_{\theta \sim \cA(Z'')} \ell(\theta,z) \right|\\
	=&\sup_{\theta,z} \ell(\theta,z)  [P_{Z'}(\theta\in S) - P_{Z''}(\theta \in S)]\\
	\leq & e^{\epsilon(Z,z')}P_{Z}(\theta\in S) +\delta((Z,z'))  - P_{Z''}(\theta \in S)\\
	\leq & (e^{\epsilon(Z,z') + \epsilon(Z,z'')} - 1)P_{Z''}(\theta \in S)  + \delta(Z,z') + \epsilon(Z,z')\delta(Z,z'')\\
	\leq&  (e^{\epsilon(Z,z') + \epsilon(Z,z'')} - 1) + \delta(Z,z') + \epsilon(Z,z')\delta(Z,z'')
	\end{align*}
	Note that the bound is independent to $z$.

	Now we will show stability implies generalization using a ``ghost sample'' trick in which we resample $Z'\sim \cD^n$ and construct $Z^{(i)}$ by replacing the $i$th data point from the $i$th data point of $Z'$.
	\begin{align*}
	&\left| \E_{Z\sim \cD^n} \left(\E_{\theta\sim \cA(Z)}\E_{z\sim\cD}\ell(\theta,z) - \E_{\theta\sim \cA(Z)}\frac{1}{n}\sum_{i=1}^n\ell(\theta,z_i) \right) \right|\\
	=&\left| \E_{Z\sim \cD^n, \{z_1',...,z_n'\}\sim\cD^n} \left(\frac{1}{n}\sum_{i=1}^n\E_{\theta\sim \cA(Z)}\ell(\theta,z_i') - \frac{1}{n}\sum_{i=1}^n\E_{\theta\sim \cA(Z^{(i)})}\ell(\theta,z_i') \right) \right|\\
	\leq&  \E_{Z\sim \cD^n, \{z_1',...,z_n'\}\sim\cD^n}\frac{1}{n}\sum_{i=1}^n\left| \E_{\theta\sim \cA(Z)}\ell(\theta,z_i') - \E_{\theta\sim \cA(Z^{(i)})}\ell(\theta,z_i')\right|\\
	\leq& \E_{Z\sim \cD^{n-1},\{z',z''\}\sim\cD^{2}} [(e^{\epsilon(Z,z') + \epsilon(Z,z'')} - 1) + \delta(Z,z') + \epsilon(Z,z')\delta(Z,z'')]
	\end{align*}
	The last step simply substitutes the stability bound. Take expectation on both sides, we get a generalization upper bound of form:
	$$
	\xi = \E_{Z\sim \cD^n} (\E_{z\sim \cD}[e^{\epsilon(Z,z)}|Z])^2-1 + \E_{Z\sim \cD^n,z\sim\cD} \delta(Z,z)
	+(\E_{Z\sim \cD^n}\E_{z\sim\cD}[e^{\epsilon(Z,z)}|Z]\E_{z\sim\cD}[\delta(Z,z)|Z].
	$$
\end{proof}
\begin{proof}[Proof of Proposition~\ref{prop:cross-gen}]
	The stability argument remains the same, because it is applied to a fixed pair of $(Z,z)$. We will modify the ghost sample arguments with and additional change of measure.
	\begin{align*}
	&\left| \E_{Z\sim \cD^n} \left(\E_{\theta\sim \cA(Z)}\E_{z\sim\cD'}\ell(\theta,z) - \E_{\theta\sim \cA(Z)}\frac{1}{n}\sum_{i=1}^n\rho(z_i)\ell(\theta,z_i) \right) \right|\\
	=&\left| \E_{Z\sim \cD^n} \left(\E_{\theta\sim \cA(Z)}\E_{z\sim\cD}\rho(z)\ell(\theta,z) - \E_{\theta\sim \cA(Z)}\frac{1}{n}\sum_{i=1}^n\rho(z_i)\ell(\theta,z_i) \right) \right|\\
	=&\left| \E_{Z\sim \cD^n, \{z_1',...,z_n'\}\sim\cD^n} \left(\frac{1}{n}\sum_{i=1}^n\E_{\theta\sim \cA(Z)}\rho(z_i')\ell(\theta,z_i') - \frac{1}{n}\sum_{i=1}^n\E_{\theta\sim \cA(Z^{(i)})}\rho(z_i')\ell(\theta,z_i') \right) \right|\\
	\leq&  \E_{Z\sim \cD^n, \{z_1',...,z_n'\}\sim\cD^n}\frac{1}{n}\sum_{i=1}^n\rho(z_i')\left| \E_{\theta\sim \cA(Z)}\ell(\theta,z_i') - \E_{\theta\sim \cA(Z^{(i)})}\ell(\theta,z_i')\right|\\
	\leq& \E_{Z\sim \cD^{n-1},\{z',z''\}\sim\cD^{2}} \rho(z'')[(e^{\epsilon(Z,z') + \epsilon(Z,z'')} - 1) + \delta(Z,z') + \epsilon(Z,z')\delta(Z,z'')]\\
	=& \E_{Z\sim \cD^{n-1},z'\sim\cD,z''\sim\cD'}[(e^{\epsilon(Z,z') + \epsilon(Z,z'')} - 1) + \delta(Z,z') + \epsilon(Z,z')\delta(Z,z'')].
	\end{align*}
\end{proof}

\begin{proof}[Proof of Corollary~\ref{cor:gen-cross}]
	$$
	\E\left[\E_{\cD}[e^{\epsilon(Z,z)}|Z]\E_{\cD'}[e^{\epsilon(Z,z)}|Z]\right] -1 + 
	\delta(1+\E[e^\epsilon(Z,z)]) \leq \sqrt{\E_\cD e^{2\epsilon}\E_{\cD'}e^{2\epsilon}} - 1  + 2\delta.$$
	The inequality uses Jensen's inequality $\E\left[\E[e^{\epsilon(Z,z)}|Z]^2\right] \leq \E e^{2\epsilon(Z,z)}$ and the monotonicity of moment generating function on non-negative random variables. The statement is obtained by Taylor's series on $\E e^{2\epsilon(Z,z)}$.
	Lastly, we use the algebraic mean to upper bound the geometric mean in the first term and then use Taylor expansion.
\end{proof}

\begin{proof}[Proof of Lemma~\ref{lem:smoothlearning}]
	By the stationarity of $\hat{\theta}$
	\begin{align*}
	\sum_i\nabla\ell(\hat{\theta},z_i) + \nabla r(\hat{\theta}) = 0
	\end{align*}
	Add and subtract $\ell(\hat{\theta},z)$ and apply first order Taylor's Theorem centered at $\hat{\theta}'$ on  $\sum_i\nabla\ell(\hat{\theta},z_i)+\nabla\ell(\hat{\theta},z_i) + \nabla r(\hat{\theta})$, we get
	$$
	\sum_i\nabla\ell(\hat{\theta}',z_i)+\nabla\ell(\hat{\theta}',z_i) + \nabla r(\hat{\theta}') + R - \nabla \ell(\hat{\theta},z) = 0.
	$$
	where if we define $\eta_t = (1-t)\hat{\theta}'+ t\hat{\theta}$, the remainder term $R\in \R^d$ can be explicitly written as
	$$
	R=\left[\int_0^1\left(\sum_i\nabla^2\ell(\eta_t,z_i) +\nabla^2\ell(\eta_t,z) + \nabla^2r(\eta_t)\right) dt \right] (\hat{\theta}-\hat{\theta}') .
	$$
	By the mean value theorem for Frechet differentiable functions, we know there is a $t$ such that we can take $\eta_t$ such that the integrand is equal to the integral.
	
	Since $\hat{\theta}'$ is a stationary point, we have
	\begin{align*}
	\sum_i\nabla\ell(\hat{\theta}',z_i) + \nabla\ell(\hat{\theta}',z) + \nabla r(\hat{\theta}') = 0
	\end{align*}
	and thus under the assumption that $\left[\int_0^1\left(\sum_i\nabla^2\ell(\eta_t,z_i) +\nabla^2\ell(\eta_t,z) + \nabla^2r(\eta_t)\right) dt \right]$ is invertible, we have
	$$
	\hat{\theta}-\hat{\theta}' = \left[\int_0^1\left(\sum_i\nabla^2\ell(\eta_t,z_i) +\nabla^2\ell(\eta_t,z) + \nabla^2r(\eta_t)\right) dt \right]^{-1} \nabla \ell(\hat{\theta},z).
	$$
	The other equality follows by symmetry.
\end{proof}

\begin{proof}[Proof of Theorem~\ref{thm:OPS}]

	Let $X'=[X;x]$, $\mathbf{y}' = [\mathbf{y};y]$. Denote $H := X^TX +\lambda I$, $H' :=(X')^TX' + \lambda I$, $g := X^T\mathbf{y}$ and $g' :=  (X')^T\mathbf{y}'$.
	Correspondingly, the posterior mean
	$\hat{\theta}= H^{-1}  g$ and $\hat{\theta}' = [H']^{-1}g'$. 
	
	
	The covariance matrix of the two distributions are $H/\gamma$ and $ H'/\gamma$. 
	Using the fact that the normalization constant is known for Gaussian, the log-likelihood ratio at output $\theta$ is
	\begin{align*}
	&\log \frac{|H^{-1}|^{-1/2} e^{-\frac{\gamma}{2}\|\theta-\hat{\theta}\|^2_{H}}}{|[H']^{-1}|^{-1/2} e^{-\frac{\gamma}{2}\|\theta-\hat{\theta}'\|^2_{H'}}}\\
	=&\underbrace{ \log \sqrt{\frac{|H|}{|H'|}}}_{(\#)}  + \frac{\gamma}{2}\underbrace{\left[ \|\theta-\hat{\theta}'\|^2_{H'} - \|\theta-\hat{\theta}\|^2_{H}\right]}_{(*)}.
	\end{align*}
	
	Note that $H'  =  H + xx^T$. By Lemma~\ref{lem:determinant},
	$$\frac{|H|}{|H'|} = \frac{|H|}{|H|(1+\mu)} = \frac{|H'|(1-\mu')}{|H'|},$$
	so 
	$$
	(\#) = \log \sqrt{(1+\mu)^{-1}}  =\log \sqrt{1-\mu'}.
	$$
	

	The second term in the above equation can be expanded into

	\begin{align}
	(*)=&\theta^T [ H' - H]\theta  + (\hat{\theta}')^T H' \hat{\theta}'  -  \hat{\theta}^TH\hat{\theta}
	-2  (\hat{\theta}')^T H'\theta + 2  \hat{\theta}^T H\theta\nonumber\\
	=&  (x^T\theta)^2 + \underbrace{(\mathbf{y}')^T X'[H']^{-1} X'^T\mathbf{y}' - \mathbf{y}^T X(H)^{-1} X^T\mathbf{y}}_{(**)} - 2 y (x^T \theta) \label{eq:linreg_ops_cal1}	 
	\end{align}
	$$
	(**)   = \left[(\mathbf{y}')^TX'[(X')^TX' + \lambda I]^{-1}X'^T\mathbf{y}'  - \mathbf{y}^TX(X^TX + \lambda I)^{-1}X^T\mathbf{y}\right]=  \left[(\mathbf{y}')^T \Pi' \mathbf{y}' - \mathbf{y}^T\Pi \mathbf{y} \right],
	$$
	where we denote the``hat'' matrices $\Pi := X(X^TX + \lambda I)^{-1}X^T$ and $\Pi'  = X'[(X')^TX' + \lambda I]^{-1}(X')^T$. Also define $v :=  X(X^TX + \lambda I)^{-1} x$. By Sherman-Morrison-Woodbury formula, we can write
	\begin{align*}
	\Pi' &= \begin{bmatrix}
	X\\
	x^T
	\end{bmatrix}\left[ H^{-1} -  H^{-1}x(1+ \mu)^{-1}x^TH^{-1}\right]\begin{bmatrix}
	X^T&  x
	\end{bmatrix}\\
	&=\begin{bmatrix}
	\Pi  -  (1+\mu)^{-1}vv^T,& v - \mu(1+\mu)^{-1}v \\
	v^T -  v^T (1+\mu)^{-1}\mu,&  \mu - \mu^2(1+\mu)^{-1}
	\end{bmatrix}	
	\end{align*}
	Note that $v^Ty  = x^T\hat{\theta}$ and $1-\mu(1+\mu)^{-1} =  (1+\mu)^{-1}$, therefore
	\begin{align*}
	(**) &= -(1+\mu)^{-1} (x^T\hat{\theta})^2 +  2 (1+\mu)^{-1} x^T\hat{\theta} + \mu(1+\mu)^{-1} y^2\\
	&= -(1+\mu)^{-1}(y-x^T\hat{\theta})^2  +  y^2.
	\end{align*}
	Substitute into \eqref{eq:linreg_ops_cal1}, we get
	$$
	(*) = (y-x^T\theta)^2 -(1+\mu)^{-1}(y-x^T\hat{\theta})^2.
	$$
	And the $\log$-probability ratio is
	\begin{align*}
	\log \frac{p(\theta|X,\mathbf{y})}{p(\theta|X',\mathbf{y}')} &= \log\sqrt{(1+\mu)^{-1}}  + \frac{\gamma}{2}\left[(y-x^T\theta)^2 -(1+\mu)^{-1}(y-x^T\hat{\theta})^2\right]\\
	&= \log\sqrt{(1+\mu)^{-1}} + \frac{\gamma}{2}\left[(x^T\hat{\theta}-x^T\theta)^2 + 2(y-x^T\hat{\theta})(x^T\hat{\theta}-x^T\theta) + \frac{\mu}{1+\mu}(y-x^T\hat{\theta})^2\right]
	\end{align*}
	Under the distribution of $\theta$ when the data is $(X,\mathbf{y})$, $x^T\theta-x^T\hat{\theta}$ follows a univariate normal distribution with mean $0$ and variance $\mu/\gamma$. By the standard tail probability of normal random variable,
	$$
	\P\left(|x^T\theta-x^T\hat{\theta}| > \sqrt{\frac{\mu}{\gamma} \log (2/\delta)}\right) \leq \frac{2 e^{-\log(2/\delta)}}{\log(2/\delta)} = \frac{\delta}{\log(2/\delta)} \explain{\leq}{\text{When $\delta<2/e$}} \delta.
	$$
	
	we can calculate $(\epsilon,\delta)$-pDP for every $\delta>0$. In particular, under $p(\theta|X,y)$
	$$
	\P(\left|\log \frac{p(\theta|X,\mathbf{y})}{p(\theta|X',\mathbf{y}')}\right| \geq \epsilon )<\delta
	$$
	for 
	$$\epsilon = \frac{1}{2}\left| -\log(1+\mu) + \frac{\mu\gamma}{(1+\mu)}(y-x^T\hat{\theta})^2\right| + \frac{\mu}{2} \log (2/\delta) + |y-x^T\hat{\theta}|\sqrt{\mu \gamma\log (2/\delta)}.
	$$
	By Lemma~\ref{lem:tailbound2DP} this implies $(\epsilon,\delta)$-DP.

	Now, we will work out an equivalent representation of the $\log$-probability ratio that depends on $\hat{\theta}'$. 
	
	Let $\mu'$ be the in-sample leverage score of $x$ with respect to $X'$, namely, $\mu' := x^T[H']^{-1} x$. By Sherman-Morrison-Woodbury formula
	\begin{align}
	H^{-1}  = [H' - xx^T]^{-1} = [H']^{-1} + [H']^{-1} x(1-\mu')^{-1}x^T [H']^{-1}.\label{eq:linreg_ops_cal2}.
	\end{align}
	
	Standard matrix algebra gives us
	\begin{align*}
	\mathbf{y}^T\Pi\mathbf{y} &= (\mathbf{y}')^TX'H^{-1}(X')^T\mathbf{y}' - yx^TH^{-1}xy - 2yx^TH^{-1}X^T\mathbf{y}\\
	&= (\mathbf{y}')^TX'H^{-1}(X')^T\mathbf{y}' - 2yx^TH^{-1}(X')^T\mathbf{y}'  + yx^TH^{-1}xy.
	\end{align*}
	Substitute \eqref{eq:linreg_ops_cal2} into the above, we get
	\begin{align*}
	\mathbf{y}^T\Pi \mathbf{y} &=(\mathbf{y}')^T \Pi'\mathbf{y}' + (1-\mu')^{-1} (x^T\hat{\theta}')^2 -2y x^T\hat{\theta}'\left[1 + \mu'(1-\mu')^{-1} \right]  + y^2 \mu' + y^2 (\mu')^2(1-\mu')^{-1}\\
	&= (\mathbf{y}')^T \Pi'\mathbf{y}' + (1-\mu')^{-1} (x^T\hat{\theta}')^2  - 2y x^T\hat{\theta}' (1-\mu')^{-1}  + y^2 (1-\mu')^{-1} -y^2
	\end{align*}
	Therefore,
	$$
	(**) =  -(y-x^T\hat{\theta}')^2 (1-\mu')^{-1} + y^2,
	$$
	and
	$$
	(*) = (y-x^T\theta)^2 -(1-\mu')^{-1}(y-x^T\hat{\theta}')^2.
	$$
	The corresponding log-probability ratio
	\begin{align*}
	\log \frac{p(\theta|X,\mathbf{y})}{p(\theta|X',\mathbf{y}')} &= -\log(\sqrt{1-\mu'}) + \frac{\gamma}{2}\left[(y-x^T\theta)^2 -(1-\mu')^{-1}(y-x^T\hat{\theta}')^2\right]\\
	&=-\log(\sqrt{1-\mu'}) + \frac{\gamma}{2}\left[(x^T\hat{\theta}'-x^T\theta)^2 + 2(x^T\hat{\theta}'-x^T\theta)(y-x^T\hat{\theta}')  -\frac{\mu'}{1-\mu'}(y-x^T\hat{\theta}')^2\right]
	\end{align*}
	
	Under the posterior distribution of $(X',\mathbf{y}')$, the mean of $x^T\theta$ is centered at $x^T\hat{\theta}'$ with variance $\mu'/\gamma$. We can then derive a tail bound of the privacy loss random variable and it implies an $(\epsilon,\delta)-pDP$ guarantee by Lemma~\ref{lem:tailbound2DP}. Specifically, it implies that the method is $(\epsilon,\delta)$-pDP with
	$$
	\epsilon = \frac{1}{2}\left|-\log(1-\mu') - \frac{\gamma\mu'}{1-\mu'}(y-x^T\hat{\theta}')^2\right| + \frac{\mu'}{2} \log(2/\delta) + \sqrt{\gamma\mu'\log(2/\delta)}|y-x^T\hat{\theta}'|.
	$$
	This complete the second statement of the proof.
\end{proof}

\begin{proof}[Proof of Proposition~\ref{prop:OPS-performance}]
	
	The proof mostly involves applying Theorem~\ref{thm:OPS} and substituting bounds over either a bounded domain assumption (typical for DP analysis), or a model assumption of how data are generated (typical for statistical analysis).
	
	\paragraph{Proof of Statement 1 in the agnostic setting.}
	For any $x\in\cX$, and any data set $X$, using the choice of regularization term, we can bound $\mu  =  1/\lambda_n$. 
	Substitute that into Theorem~\ref{thm:OPS}, and use the inequality that $\log(1+x)\leq x$ we get the first expression. 
	
	Now, restricting ourselves to the bounded domain. Under the choice of $\lambda_n$, we can choose an $X$,$\mathbf{y}$ with a singular value equal to $\sqrt{\lambda_n}$ and the corresponding singular vector $v\in \{-1,1\}^n$ such that the following upper bounds are attained
	$$
	\|(X^TX + \lambda_nI)^{-1}X^T\| \leq \frac{1}{2\sqrt{\lambda_n}}.
	$$
	$$
	\|\hat{\theta}\| \leq \|(X^TX + \lambda_nI)^{-1}X^T\|\|y\| \leq \frac{\sqrt{n}}{2\sqrt{\lambda_n}}.
	$$
	Now choose $(x,y)$ such that $|x^T\hat{\theta}| = \|x\|\|\hat{\theta}\|$, we get that $\sup_{(X,\mathbf{y}), (x,y)}|y-\hat{\theta}^Tx| = 1+\frac{\sqrt{n}}{2\sqrt{\lambda_n}}$.
	The DP claim follows by substituting the upper bound into the pDP's expression.

	\paragraph{Proof of Statement 2 under the model assumption.}
	To prove the second claim, note that by Assumption (b)(d), the smallest eigenvalue of $X^TX$ is lower bounded by $d/nm$.
	Also under the model assumption, the ridge regression estimator concentrates around $\theta_0$.
	
	
	In particular, under the model assumption, the ridge regression estimate
	\begin{equation}\label{eq:thetahat}
	\hat{\theta} = (X^TX +\lambda_n I)^{-1}X^T y = (X^TX +\lambda_n I)^{-1}X^TX\theta_0 + (X^TX +\lambda_n I)^{-1}X^TZ
	\end{equation}
	With high probability over the distribution of $Z$
	$$
	\|\hat{\theta}-\theta_0\|^2 =O(\frac{d\sigma^2 \log(n)}{n} + \frac{\lambda_n^2d^2\|\theta_0\|^2}{n^2} ),
	$$
	thus for all $(x,y)$ satisfying $\|x\|\leq 1$ $y\leq 1$, we get 
	$$|y-x^T\hat{\theta}| \leq |y-x^T\theta_0| + |x^T(\hat{\theta}-\theta_0)| = O(1+\|\theta_0\|).$$
	Under the assumption that $n>10d\log n$, $\|\theta_0\|=O(1)$ and $\sigma=O(1)$ this is effectively a constant.
	
	For $x\in \cX$ and $y\sim \cN(\theta_0^Tx,\sigma^2)$, using standard Gaussian tail bound, with high probability the perturbation is bounded, therefore $|y-x^T\theta_0|^2\leq \sigma^2\log(2/\delta')$. 
	
	Lastly, we address the case of the average pDP loss over the empirical data distribution. Besides taking into the above bound on $|y-x^T\hat{\theta}|$, we further consider adding the different parts over the distributions. Since this is to deal with data points in the data set, we will instantiate the bound \eqref{eq:eps_master2}. Our assumption on $\gamma_n$, $\lambda_n$ ensures that the dominant term is the third term, thus
	$$
	\frac{1}{n}\sum_{i=1}^n \epsilon_n((X,\mathbf{t}),(x_i,y_i))^2 \leq C\gamma_n \frac{1}{n}\sum_{i=1}^n(y_i-x_i^T\hat{\theta}')^2 x_i^T(X^TX+\lambda_nI)^{-1}x_i.
	$$
	Under the high probability event that the noise is bounded by $\sigma\sqrt{2/\delta'}$ for all data points, we can extract them out then note that
	$$\frac{1}{n}\sum_{i=1}^nx_i^T(X^TX)^{-1}x_i = \frac{1}{n}\tr(\sum_{i=1}^nx_ix_i^T (X^TX+\lambda_n I)) \leq  \frac{1}{n}\tr(I) = \frac{d}{n}.$$
	
	Substitute these bounds into Theorem~\ref{thm:OPS}, and we obtain the Statement 2.

	\paragraph{Proof of Statement 3 under the model assumption.}
	By \eqref{eq:thetahat} and the fact that OPS can be thought of as adding an independent multivariate Gaussian noise with covariance matrix $(X^TX +\lambda_n I)^{-1}X^TX(X^TX +\lambda_n I)^{-1}/\gamma_n$, we get 
	$$
	\tilde{\theta} =  (X^TX +\lambda_n I)^{-1}X^TX\theta_0 + \sqrt{1+\gamma_n}(X^TX +\lambda_n I)^{-1}X^TZ.
	$$
	By a bias-variance decomposition, we get
	\begin{align*}
	&\E(\|\tilde{\theta}-\theta_0\|_2^2|X) = \Var(\tilde{\theta}|X) + \|\E\tilde{\theta}-\theta_0\|^2\\
	=& (1+ \gamma_n^{-1})\sigma^2\tr\left[(X^TX +\lambda_n I)^{-1}X^TX(X^TX +\lambda_n I)^{-1}\right] + \left\|\left[I-(X^TX +\lambda_n I)^{-1}X^TX\right]\theta_0\right\|^2\\
	=& (1+\gamma_n^{-1})\sigma^2\sum_{i=1}^d \frac{\sigma_i^2}{(\sigma_i^2+\lambda_n)^2}  + \lambda_n^2\theta_0^T(X^TX+\lambda_nI)^{-2}\theta_0\\
	\leq&(1+\gamma_n^{-1})\sigma^2\tr(X^TX+\lambda_n I)^{-1}  + \lambda_n^2m^{-2}n^{-2}\|\theta_0\|^2
	\end{align*}
	The proof is complete by substitute the values of $\gamma_n$ and $\lambda_n$ into the inequality and noting that $m=\Omega(1)$ and under the model assumption $\|\theta_0\|$ does not grow with $n$. Clearly, if $\gamma_n = \omega(1)$ and $\lambda_n=o(\sqrt{n})$, then the algorithm is asymptotically efficient.
\end{proof}

\begin{proof}[Proof of Proposition~\ref{prop:AdaOPS}]
	We will first prove the claim on differential privacy and then analyze the statistical efficiency.	
	\paragraph{Proof of differential privacy.} 
	First of all, by Weyl's theorem, and the assumption that $\|xx^T\|_2\leq 1$, we get that the global sensitivity of $\lambda_{\min}(X^TX)$ is $1$. We will use $\lambda_{\min}$ as the short hand of $\lambda_{\min}(X^TX)$ in the rest of the proof.
	So releasing $\tilde{\lambda}_{\min}$ is $(\epsilon/2,\delta/2)$-DP using the standard Gaussian mechanism.
	Secondly, under the same event with probability at least $1-\delta/2$, we have 
	$$\lambda_{\min} - \frac{\log(4/\delta)}{\epsilon}\leq \tilde{\lambda}_{\min} \leq \lambda_{\min} + \frac{\log(4/\delta)}{\epsilon}.$$
	Therefore, by our selection rule of the regularization parameter $\lambda_n$,
	$$\frac{n}{d\kappa} \leq \lambda_{\min}(X^TX+\lambda_n I) \leq \max\{\lambda_{\min},\frac{n}{d\kappa} + \frac{\log(4/\delta)}{\epsilon}\}.$$
	The lower bound implies that for any $(x,y)$ satisfying the condition, the out of sample leverage score
	\begin{equation}\label{eq:pf_dp_leverage}
	\mu = x^T(X^TX+\lambda_n I)^{-1}x  \leq \frac{\kappa d}{n}.
	\end{equation}
	It also implies an upper bound on the prediction error:
	\begin{equation}\label{eq:pf_dp_predict}
	|y-x^T\hat{\theta}| \leq 1 + \|\hat{\theta}\| \leq 1 + \|(X^TX + \lambda_n I)^{-1}X^T\|_2\|y\|_2 \leq \min\sqrt{2d\kappa}.
	\end{equation}
	We will prove the final inequality above using the following lemma with $h=n/d\kappa$ and then note that $\|y\|_2\leq \sqrt{n}$.
	\begin{lemma}\label{lem:eigenvalues_of_regularized_pseudoinverse}
		For any matrix $X$, and any $\lambda\geq0$. If $\lambda_{\min}(X^TX + \lambda I) \geq h$, then $$\|(X^TX + \lambda I)^{-1}X^T\|\leq \sqrt{2/h}.$$
	\end{lemma}
	The proof is technical so we defer it to later.
	
	Now combine \eqref{eq:pf_dp_leverage}\eqref{eq:pf_dp_predict} with Theorem~\ref{thm:OPS}, we get that the OPS step which obeys
	$(\tilde{\epsilon},\delta/2)$-pDP with
	\begin{align*}
	\tilde{\epsilon}((X,\mathbf{y}), (x,y)) &\leq \frac{\mu}{2}(1+\log(4/\delta))  + \frac{1}{2}\gamma_n\min(\mu,1)(y-x^T\hat{\theta})^2 + \sqrt{\gamma\mu\log(4/\delta)}|y-x^T\hat{\theta}|\\
	& \leq \frac{\kappa d(1+\log(4/\delta))}{2n} + \frac{\gamma_n}{2}\frac{\kappa d}{n} 2\kappa d +  \sqrt{\frac{\gamma_n}{2}\frac{\kappa d}{n} 2\kappa d \log(4/\delta)}\\
	&\leq \epsilon/8 + \epsilon/8 + \epsilon/4 \leq \epsilon/2
	\end{align*}
	Note that in the last step, we made use of the choice of $\gamma_n$ and the condition that concerns $\epsilon$ and $\kappa$ as stated in the algorithm. Since this upper bound holds for all data set $(X,\mathbf{y})$ and all privacy target $(x,y)$. The OPS algorithm also satisfies $(\epsilon/2,\delta/2)-DP$.
	
	The proof of the first claim is complete when we compose the two data access.
	
	\paragraph{Proof of the statistical efficiency.} Now we switch gear to analyze the estimation error bound.
	Let event $E$ be the event that $\tilde{\lambda}_{\min} > \lambda_{\min} - \frac{\sqrt{10\log(n)}\sqrt{\log(4/\delta)}}{\epsilon/2}$, which happens with probability $1-n^{-10}$. Under $E$, we have $\lambda_0=0$. By our assumption, this happens with 
	
	Applying the third claim in Proposition~\ref{prop:OPS-performance}, we get that 
	$$
	\E[\|\tilde{\theta}-\theta_0\|^2|X,E] \leq (1+\gamma_n^{-1})\sigma^2\tr((X^TX)^{-1}).
	$$
	Under the small probability event $E^c$, we use a crude upper bound that takes the sum of the maximum square bias and maximum variance.
	$$
	\E[\|\tilde{\theta}-\theta_0\|^2|X,E^c]\leq (1+\gamma_n^{-1})\sigma^2\tr((X^TX)^{-1})+ \|\theta_0\|^2
	$$
	
	by law of total expectation, for an event $E$
	\begin{align*}
	\E[\|\tilde{\theta}-\theta_0\|^2|X] &= \E[\|\tilde{\theta}-\theta_0\|^2|X,E]\P(E|X) + \E[\|\tilde{\theta}-\theta_0\|^2|X,E^c]\P(E^c|X)\\
	&\leq (1+\gamma_n^{-1})\sigma^2\tr((X^TX)^{-1}) + \P(E^c)\|\theta_0\|^2 =(1+\gamma_n^{-1})\sigma^2\tr((X^TX)^{-1}) + O(n^{-10}).
	\end{align*}
	The proof is complete by substituting $\gamma_n$ into the bound.
\end{proof}
\begin{proof}[Proof of Lemma~\ref{lem:eigenvalues_of_regularized_pseudoinverse}.]
	Take SVD of $X = U\Sigma V^T$, we can write 
	\begin{align*}
	\|(X^TX + \lambda_n I)^{-1}X^T\|_2 &= \max_{i\in[d]} \frac{\Sigma_{ii}}{\Sigma_{ii}^2 + \lambda_n} 
	\end{align*}
	We now discuss two cases. First, for those $i\in[d]$ such that $\Sigma_{ii}^2 \leq \lambda_n$. In this case, adding $\lambda_n$ on both sides ensures that
	$$
	h\leq \lambda_{\min}(X^TX+\lambda_n I) = \lambda_{\min} + \lambda_n  \leq \Sigma_{ii}^2 + \lambda_n\leq  2\lambda_n.
	$$
	and therefore if $\Sigma_{ii}>0$
	\begin{equation}\label{eq:operator_norm_arg1}
	\frac{\Sigma_{ii}}{\Sigma_{ii}^2 + \lambda_n} = \frac{1}{\Sigma_{ii} + \lambda_n/\Sigma_{ii}}\leq \frac{1}{2\sqrt{\lambda_n}}\leq \sqrt{1/(2h)}.
	\end{equation}
	The final inequality is also true for $\Sigma_{ii}=0$.
	If on the other hand, for those $i\in[d]$ such that, $\Sigma_{ii}^2 > \lambda_n$. This time by adding $\Sigma_{ii}^2$ on both sides, we get
	$$
	2\Sigma_{ii}^2  > \lambda_n + \Sigma_{ii}^2 \geq \lambda_n + \lambda_{\min} = \lambda_{\min}(X^TX+\lambda_n I)  \geq \frac{n}{\kappa}.
	$$
	This implies that
	\begin{equation}\label{eq:operator_norm_arg2}
	\frac{\Sigma_{ii}}{\Sigma_{ii}^2 + \lambda_n}\leq \frac{1}{\Sigma_{ii}}\leq \sqrt{2/h}.
	\end{equation}
	Combine \eqref{eq:operator_norm_arg1} and \eqref{eq:operator_norm_arg2} we get
	$$\|(X^TX + \lambda_n I)^{-1}X^T\|_2 \leq \sqrt{2/h}$$
	
\end{proof}

\section{Technical lemmas}

\begin{lemma}\label{lem:error_expansion}
	Let $\hat{\theta}' = (X^TX + E_1)^{-1}(X\mathbf{y} + E_2)$ for any matrix $E_1$, $E_2$.
	$$
	\hat{\theta}' - \hat{\theta} =   (X^TX + E_1)^{-1} (E_2-E_1 \hat{\theta})
	$$
\end{lemma}
\begin{proof}
	\begin{align*}
	\hat{\theta}' =& (X^TX + E_1)^{-1} (X^T\mathbf{y}  + E_2)\\
	=& (X^TX + E_1)^{-1} (X^TX) (X^TX)^{-1}X^T\mathbf{y} + (X^TX + E_1)^{-1} E_2\\
	=& \hat{\theta} + \left[ (X^TX + E_1)^{-1} (X^TX + E_1) - (X^TX + E_1)^{-1} E_1 - I_d  \right] \hat{\theta} + (X^TX + E_1)^{-1} E_2\\
	=&  \hat{\theta} - (X^TX + E_1)^{-1} E_1 \hat{\theta} + (X^TX + E_1)^{-1} E_2\\
	=& \hat{\theta} + (X^TX + E_1)^{-1} (E_2-E_1 \hat{\theta})
	\end{align*}
\end{proof}

\begin{lemma}[Sherman-Morrison-Woodbury Formula]\label{lem:woodbury}
	Let $A,U,C,V$ be matrices of compatible size, assume $A,C$ and $C^{-1}+VA^{-1}U$ are all invertible, then
	$$
	(A+ UCV)^{-1} = A^{-1} - A^{-1}U(C^{-1}+VA^{-1}U)^{-1}A^{-1}.
	$$
\end{lemma}
\begin{lemma}[Determinant of Rank-1 perturbation]\label{lem:determinant}
	For invertible matrix $A$ and vector $c,d$ of compatible dimension
	$$\det(A + cd^T)  = \det(A)(1+d^TA^{-1}c).$$
\end{lemma}

\begin{lemma}[Weyl's eigenvalue bound {\citep[Theorem~1]{stewart1998perturbation}}]
	Let $X,Y,E\in \R^{m\times n}$, w.l.o.g., $m\geq n$. If $X-Y = E$, then $|\sigma_i(X)-\sigma_i(Y)|\leq \|E\|_2 $ for all $i=1,...,n$.
\end{lemma}

\begin{lemma}[Gaussian tail bound]
	Let $X\sim \cN(0,1)$. Then 
	$$
	\P(|X| >\epsilon) \leq \frac{2e^{-\epsilon^2/2}}{\epsilon}.
	$$
\end{lemma}

\begin{lemma}[Tail bound to $(\epsilon,\delta)$-DP conversion]\label{lem:tailbound2DP}
	Let $\epsilon(\theta) = \log(\frac{p(\theta)}{p'(\theta)})$ where $p$ and $p'$ are densities of $\theta$. If 
	$$
	\P(|\epsilon(\theta)| > t) \leq \delta
	$$
	then for any measurable set $\cS$
	$$
	\P_p(\theta \in \cS)  \leq e^t \P_{p'}(\theta \in \cS) + \delta.
	$$
	and 
	$$\P_{p'}(\theta \in \cS)  \leq e^t \P_{p}(\theta \in \cS) + \delta$$
\end{lemma}
\begin{proof}
	Since $\log(\frac{p(\theta)}{p'(\theta)}) = - \log(\frac{p'(\theta)}{p(\theta)})$ and the tail bound is two-sided. It suffices for us to prove just one direction. Let $E$ be the event that $|\epsilon(\theta)| > t$.
	$$
	\P_p(\theta \in \cS)  =  \P_p(\theta \in \cS \cup E^c) + \P_p(\theta \in \cS \cup E) \leq  \P_{p'}(\theta \in \cS \cup E) e^t + \P_p(\theta\in E)\leq e^t\P_{p'}(\theta \in \cS) + \delta.
	$$
\end{proof}
\begin{lemma}[Matrix Hoeffding inequality \citep{mackey2014matrix}]
	Consider a finite sequence $X_1,...,X_n$ of independent random and self-adjoint matrices with dimension $d$ and $A_1,...,A_n$ be a sequence of fixed self-adjoint matrices. In addition, let $\E X_i = 0$ and $X_i^2\preceq A_i^2$ almost surely for all $i=1,...,n$. Then, for all $t\geq 0$
	$$\P\left\{ \lambda_{\max}(\sum_{i=1}^n X_i) \geq t \right\} \leq d e^{-t^2/2\sigma^2}$$
	where $\sigma^2\leq \|\sum_{i=1}^n A_i^2\|$.
\end{lemma}

\end{document}